\documentclass{article}

\usepackage[final]{neurips_2025}

\usepackage[utf8]{inputenc} 
\usepackage[T1]{fontenc}    
\usepackage{hyperref}       
\usepackage{url}            
\usepackage{booktabs}       
\usepackage{amsfonts}       
\usepackage{nicefrac}       
\usepackage{microtype}      
\usepackage{xcolor}         

\usepackage{subcaption}
\usepackage{graphicx}
\usepackage{wrapfig}
\usepackage{amsmath,amssymb,amsthm}
\usepackage{thmtools}  
\usepackage{bm,bbm}
\usepackage{accents}
\usepackage[linesnumbered,ruled,vlined]{algorithm2e}
\usepackage[colorinlistoftodos]{todonotes}
\usepackage{minitoc}  

\newcommand{\E}{\mathbb{E}}
\newcommand{\R}{\mathbb{R}}
\newcommand{\bR}{\mathbb{R}}
\newcommand{\bN}{\mathbb{N}}
\newcommand{\bE}{\mathbb{E}}
\newcommand{\bP}{\mathbb{P}}

\newcommand{\cA}{\mathcal{A}}
\newcommand{\cB}{\mathcal{B}}

\newcommand{\cD}{\mathcal{D}}
\newcommand{\cE}{\mathcal{E}}
\newcommand{\cF}{\mathcal{F}}

\newcommand{\cL}{\mathcal{L}}
\newcommand{\cM}{\mathcal{M}}
\newcommand{\cN}{\mathcal{N}}
\newcommand{\cO}{\mathcal{O}}
\newcommand{\cP}{\mathcal{P}}

\newcommand{\cS}{\mathcal{S}}
\newcommand{\cT}{\mathcal{T}}

\newcommand{\cX}{\mathcal{X}}

\newcommand{\cZ}{\mathcal{Z}}

\newcommand{\thetahat}{\hat{\theta}}
\newcommand{\thetatilde}{\tilde{\theta}}

\DeclareMathOperator*{\argmax}{arg\,max}
\DeclareMathOperator*{\argmin}{arg\,min}

\newcommand{\eqdef}{:=}
\newcommand\numberthis{\addtocounter{equation}{1}\tag{\theequation}}

\newcommand{\ind}{\mathbbm{1}}

\newcommand{\tr}{\operatorname{tr}}

\DeclareMathOperator{\dom}{\operatorname{dom}}

\newcommand{\diff}{\mathop{}\!\mathrm{d}}

\theoremstyle{plain}
\newtheorem{theorem}{Theorem}[section]
\newtheorem{proposition}[theorem]{Proposition}
\newtheorem{lemma}[theorem]{Lemma}

\theoremstyle{definition}

\newtheorem{assumption}[theorem]{Assumption}
\theoremstyle{remark}
\newtheorem{remark}[theorem]{Remark}

\newcommand{\bc}[1]{\left\{{#1}\right\}}
\newcommand{\br}[1]{\left({#1}\right)}
\newcommand{\bs}[1]{\left[{#1}\right]}
\newcommand{\abs}[1]{\left| {#1} \right|}

\newcommand{\norm}[1]{\left\lVert#1\right\rVert}
\newcommand\ip[2]{\langle #1, #2 \rangle}

\newcommand{\Pp}{\mathbb{P}}
\newcommand{\Ppi}{\Pp_{\pi}}

\newcommand{\effH}{H_{\gamma}}

\newcommand{\pihat}{\hat{\pi}}

\newcommand{\ts}{t_{\operatorname{stop}}}

\newcommand{\ie}{\textit{i.e.} }
\newcommand{\eg}{\textit{e.g.} }

\newcommand{\lalgo}{\texttt{LRPO-OD-Regret}}

\newcommand{\oracle}{\texttt{A}^{\texttt{PAC}}_{\texttt{RL}}}
\newcommand{\subopt}{\mathsf{SubOpt}}
\DeclareRobustCommand{\bigO}{\mathcal{O}}
\DeclareRobustCommand{\bigOt}{\widetilde{\mathcal{O}}}

\title{Efficient Preference-Based Reinforcement Learning: Randomized Exploration Meets Experimental Design}

\author{%
  Andreas Schlaginhaufen\thanks{Correspondence to \texttt{andreas.schlaginhaufen@epfl.ch}.}\\
  \text{SYCAMORE, EPFL}
  \And
  Reda Ouhamma\\
  \text{SYCAMORE, EPFL}
  \And
  Maryam Kamgarpour\\
  \text{SYCAMORE, EPFL}
}

\begin{document}
\doparttoc 
\faketableofcontents 
\part{} 

\maketitle

\begin{abstract}
    We study reinforcement learning from human feedback in general Markov decision processes, where agents learn from trajectory-level preference comparisons. A central challenge in this setting is to design algorithms that select informative preference queries to identify the underlying reward while ensuring theoretical guarantees. We propose a meta-algorithm based on randomized exploration, which avoids the computational challenges associated with optimistic approaches and remains tractable. We establish both regret and last-iterate guarantees under mild reinforcement learning oracle assumptions. To improve query complexity, we introduce and analyze an improved algorithm that collects batches of trajectory pairs and applies optimal experimental design to select informative comparison queries. The batch structure also enables parallelization of preference queries, which is relevant in practical deployment as feedback can be gathered concurrently. Empirical evaluation confirms that the proposed method is competitive with reward-based reinforcement learning while requiring a small number of preference queries.
\end{abstract}

\section{Introduction}
\label{sec:introduction}
Reinforcement learning (RL) is a fundamental paradigm in machine learning, where agents learn to make sequential decisions by interacting with an environment to maximize cumulative rewards \citep{barto2021reinforcement}. RL has enabled advances in domains such as game play \citep{silver2017mastering}, robotics \citep{todorov2012mujoco}, or autonomous driving \citep{lu2023imitation}.
However, the practicality of RL is hindered by the challenge of designing rewards: crafting a reward function that aligns with human objectives is often difficult, and a misspecified reward function can lead to suboptimal or unsafe behavior \citep{amodei2016concrete, hadfield2017inverse}. This motivates the development of principled alternatives to manual reward design. 

Rather than relying on manually specified reward functions, reinforcement learning from human feedback (RLHF) guides learning through preference feedback: at each step, a human oracle compares trajectories and indicates which is preferable \citep{christiano2017deep}. This preference signal is often much easier to provide than engineering a reward function \citep{pereira2019online, lee2023aligning}. RLHF has proven to be effective in robotics \citep{jain2013learning} and, more recently, finetuning of large language models \citep{ziegler2019fine, stiennon2020learning, rafailov2023direct}. This highlights the practical relevance of RLHF compared to reward-based learning.


Despite its empirical success, the theoretical foundations of RLHF are still in development. Existing works first studied the simpler setting of dueling bandits. In this context, the learner selects pairs of actions and observes noisy preference feedback \citep{yue2012k,komiyama2015regret}. Classical algorithms for regret minimization in this setting include approaches based on zeroth-order optimization \citep{yue2009interactively} or the principle of optimism \citep{ailon2014reducing}. A key challenge in this setting is reducing the number of preference queries. For this purpose, several recent works propose query selection strategies for dueling bandits \citep{das2024active, liu2024dual, scheid2024optimal, mukherjee2024optimal}, often hinging on optimal experimental design mechanisms \citep{pukelsheim2006optimal}. 
However, such approaches are usually limited to finite-armed bandits, where the resulting optimization problems can be solved efficiently.

In the online RL setting, the theory of RLHF has received increasing attention, with several works establishing either regret or probably approximately correct (PAC) guarantees. PAC-RL methods aim to identify a near-optimal policy with high probability \citep{xu2020preference, novoseller2020dueling, zhu2023principled}, while regret-based approaches provide bounds on the cumulative reward during learning \citep{pacchiano2021dueling, chen2022human, wu2023making}. Similar to dueling bandits, a central challenge in RLHF is to actively select informative trajectory comparisons to drive learning. In RL, however, this active-learning problem presents additional difficulties: First, the learner cannot freely choose arbitrary state-action pairs or trajectories, but must reach them through exploration \citep{wagenmaker2022beyond}. Second, many existing approaches with guarantees \citep{pacchiano2021dueling, zhan2024provablereward} rely on maximizing an exploration bonus involving a norm of state distributions -- a problem which is computationally intractable even in tabular settings \citep{efroni2021reinforcement}.\looseness-1

To sidestep these challenges, another line of work focuses on RLHF with offline data. In this setting, learning proceeds over a fixed pre-collected dataset of trajectory preferences \citep{zhu2023principled, zhan2023provable}. Although this offline paradigm avoids the need for online exploration and active query selection, it depends critically on having access to sufficiently diverse and informative preference data a priori \citep{rashidinejad2021bridging, xie2021policy, zanette2021provable, zanette2023realizability}  — a requirement that can be difficult to meet in practice. Hence, this merely shifts the exploration and active learning challenges to the data collection phase.

Despite progress on statistical guarantees in RLHF, a central challenge remains open: designing tractable algorithms for active preference query selection that reduce the workload of human annotators. In existing approaches, a human must provide feedback at every round, which is impractical in real-world applications. Our goal in this work is to develop RLHF algorithms that are computationally efficient, reduce the demand for human feedback, and actively select informative queries. Some recent work has made progress on computational tractability. For instance, \citet{wu2023making} propose a randomized exploration algorithm with regret guarantees limited to linear dynamics, and \citet{dwaracherla2024efficient} show empirically that randomized exploration is efficient for finetuning of language models. In parallel, \citet{wang2023rlhf} introduce a general reduction from RLHF to standard RL and establishes PAC-style guarantees under RL oracle access. However, neither approach addresses the open challenges of reducing feedback requirements or enabling active query selection. \looseness-1

\paragraph{Contributions} In this work, we focus on reinforcement learning from human feedback (RLHF) and develop meta-algorithms that reduce the RLHF problem to standard RL by leveraging existing RL algorithms as subroutines.
Leveraging randomized exploration for tractable and efficient preference query selection, we provide both online algorithms with regret guarantees and a preference-free algorithm with PAC-style guarantees under RL-oracle assumptions. Our contributions are as follows:\looseness-1
\begin{itemize}
    \item We propose two meta-algorithms for RLHF using RL oracles: Algorithm~\ref{alg:regret}, optimized for regret minimization, and Algorithm~\ref{alg:explore}, which performs preference-free exploration and defers preference collection to a single batch at the end. For these methods, we establish regret and PAC-style guarantees, respectively, holding in general MDPs.
    \item We present a second meta-algorithm for regret minimization (Algorithm~\ref{alg:lazy_thompson_sampling}) with better scalability and query efficiency thanks to: Lazy updates, inspired by linear bandits \citep{abbasi2011improved}, which enables parallelization of the preference oracle calls; Greedy optimal design, which selects informative preference queries and improves sample efficiency.\looseness-1
    \item We provide empirical results showing that: Our algorithms are implementable, competitive with reward-based RL, and substantially outperform a baseline that relies solely on entropy-based exploration; Algorithm~\ref{alg:lazy_thompson_sampling} achieves comparable performance to Algorithm~\ref{alg:regret} while significantly reducing the query complexity.\footnote{The code is openly accessible at \url{https://github.com/andrschl/isaac_rlhf}.}
\end{itemize}

\section{Preliminaries}
\label{sec:preliminaries}

\paragraph{Notation} Let $\bN$ and $\R$ denote the sets of natural and real numbers, respectively.
We write $\norm{\cdot}$ for the Euclidean norm and $\langle \cdot,\cdot\rangle$ for the standard inner product in $\R^d$. Moreover, for a positive definite matrix $A\in\R^{d\times d}$, we denote $\norm{x}_A := \sqrt{\ip{x}{Ax}}$ for the Mahalanobis norm. Furthermore, we denote the closed Euclidean ball of radius $a>0$ by $\cB^d(a)\subset\R^d$, and for a compact subset $\cX\subset\R^n$, we denote the set of all probability measures supported on $\cX$ by $\Delta_{\cX}$. Finally, we use the standard notation $\cO(n)$ and $\Omega(n)$ for asymptotic upper and lower bounds, as well as $\tilde\cO(n)=\cO(n\operatorname{polylog(n)})$ for suppressing polylogarithmic terms.

\paragraph{Setting} We consider an infinite-horizon\footnote{Our results extend directly to the finite-horizon setting as well.} Markov decision process (MDP) $\cM=\{\cS, \cA, \nu_0, P, r, \gamma\}$ with state and action spaces $\cS \subseteq \bR^n$ and $\cA \subset \bR^m$, respectively, initial state distribution $s_0\sim\nu_0$, transition law $s_{h+1}\sim P(\cdot|s_h, a_h)$, and discount rate $\gamma\in(0,1)$. We assume a linear reward model $r_{\theta^*}(s,a):=\ip{\theta^*}{\phi(s,a)}$, where $\norm{\theta^*}\leq B$ and $\phi:\cS\times\cA\to\R^d$ is a continuous feature mapping such that $\max_{s, a} \norm{\phi(s,a)}\leq L$. We denote the set of all trajectories as $\cT \eqdef {(\cS\times\cA)}^\infty$, and the distribution over $\cT$ induced by a stationary Markov policy $\pi:\cS\to\Delta_{\cA}$ as $\Ppi$. For a trajectory $\tau = (s_0,a_0,s_1, \hdots)\in \cT$, we denote the discounted sum of features by $\phi(\tau):= \sum_{h = 0}^{\infty}\gamma^h\phi(s_h, a_h)$ and the feature expectation of a policy $\pi$ by $\phi(\pi):= \E_{\tau\sim\Ppi}\bs{\phi(\tau)}$. Furthermore, given a reward parameter $\theta$, we denote the value of a policy $\pi$ by $V_{\theta}^{\pi}:= \E_{\tau\sim\Ppi}\bs{\sum_{h=0}^{\infty}\gamma^h r_{\theta}(s_h, a_h)}=\ip{\theta}{\phi(\pi)}$ and the optimal value by $V_{\theta}^*:=\max_{\pi}V_{\theta}^{\pi}$.

\paragraph{Interaction protocol} For each round \(t=1,\ldots, T\) of RLHF, a learner, the MDP, and a preference oracle interact as follows. The learner selects two policies $\pi_t$ and $\pi'_t$, and executes them to obtain two trajectories $\tau_t\sim \bP_{\pi_t}$ and $\tau_t'\sim \bP_{\pi_t'}$. Subsequently, the learner may query the preference oracle, which returns a binary label $y_t = \ind\br{\tau_t \succ \tau'_t} \in \bc{0,1}$. The label equals one if the trajectory $\tau_t$ is preferred over $\tau'_t$, denoted as $\tau_t \succ \tau'_t$, and zero otherwise. Each such interaction is one RLHF round.

\paragraph{Preference model} We consider a stochastic preference model characterized by a preference function $\cP:\cT\times \cT\to [0,1]$, assigning to each pair of trajectories $\tau, \tau'\in\cT$ the probability $\cP(\tau\succ\tau')$ of preferring $\tau$ to $\tau'$. We make the following assumption about the preference model.
\begin{assumption}[Bradley–Terry model]
    The preference function $\cP$ satisfies for all $\tau,\tau'\in \cT$
    \begin{equation}\label{eq:bt_model}
        \cP(\tau\succ \tau') = \sigma\br{\ip{\theta^*}{\phi(\tau) - \phi(\tau')}},
    \end{equation}
    where $\sigma(x) = 1/(1+e^{-x})$ denotes the sigmoid function.
\end{assumption}
This preference model is a special case of the Plackett-Luce model \citep{plackett1975analysis,luce1959individual}, and is commonly used in the dueling bandit setting as well as the RLHF framework \citep{yue2012k, christiano2017deep, ouyang2022training}. It captures that the probability of preferring $\tau$ over $\tau'$ is increasing in the difference between their values $\ip{\theta^*}{\phi(\tau) - \phi(\tau')}$.

\begin{remark}
    In practice, we cannot compare trajectories of infinite length. Fortunately, many environments terminate in finite time, and otherwise one may truncate each trajectory at horizon $H=\cO(\log_{\gamma}(\varepsilon))$, introducing at most an $\varepsilon$ error in value estimates (see \eg, \citep{schlaginhaufen2024towards}). For simplicity, however, we omit this truncation step in our presentation.
\end{remark}

\paragraph{Regret} To assess the learner's online performance, we consider the cumulative regret
\begin{equation*}
    R(T) = \sum_{t=1}^T \frac{\br{V_{\theta^*}^*-V_{\theta^*}^{\pi_t}} + (V_{\theta^*}^*-V_{\theta^*}^{\pi'_t})}{2} = \frac{1}{2}\sum_{t=1}^T \br{2V_{\theta^*}^* - V_{\theta^*}^{\pi_t} - V_{\theta^*}^{\pi_t'}}.
\end{equation*}
Cumulative regret has been widely adopted in the RL and RLHF literature \citep{abbasi2011improved, zanette2020frequentist, wang2023rlhf, zhan2024provablereward}. However, cumulative regret doesn't provide us with a guarantee of the last iterate's suboptimality. As a second metric, we therefore also consider the suboptimality of an output policy.

\paragraph{Suboptimality} The suboptimality of an output policy $\pihat$ is defined by
\begin{equation*}
    \subopt(\hat{\pi}) := V_{\theta^*}^*-V_{\theta^*}^{\hat{\pi}}.
\end{equation*}
Suboptimality has previously been considered as a performance metric for offline RLHF \citep{zhu2023principled}, contextual bandits \citep{das2024active}, and online RLHF \citep{wang2023rlhf}. In the following, we propose a meta-algorithm that features two variants: one with theoretical guarantees on cumulative regret, and another specifically ensuring a bound on last-iterate suboptimality.

\vspace{-0.1cm}
\section{Randomized Preference Optimization}
\label{sec:randomized_preference_optimization}
\vspace{-0.1cm}

\vspace{-0.1cm}

\subsection{Algorithm}
\vspace{-0.1cm}
Both algorithms share the same key ingredients: (i) estimating the reward parameter from preference feedback using maximum likelihood estimation, (ii) sampling reward parameters from a Gaussian distribution, reminiscent of linear Thompson sampling \citep{pmlr-v54-abeille17a}, and (iii) leveraging an RL oracle to find an approximately optimal policy for the sampled reward parameter. In Algorithm~\ref{alg:regret}, we apply these steps iteratively and, at each round, query the preference oracle by comparing a trajectory from the new policy to one from the previous policy to minimize regret. In contrast, Algorithm~\ref{alg:explore} performs preference-free exploration and defers all preference queries to a single batch at the end.
\vspace{-0.1cm}

\paragraph{Maximum likelihood estimation} Considering our preference model \eqref{eq:bt_model}, a standard approach for estimating the reward parameter $\theta^*$ is via maximum likelihood estimation. Given a pair of trajectories $\tau_k=(s_{h,k}, a_{h,k})_{h=0}^\infty$ and $\tau_k'=(s'_{h,k}, a'_{h,k})_{h=0}^\infty$ we consider the design points $x_k := \phi(\tau_k) - \phi(\tau'_k) = \sum_{h=0}^{\infty} \gamma^h (\phi(s_{h, k}, a_{h,k})-\phi(s'_{h,k}, a'_{h,k}))$ and the preference labels $y_k = \ind(\tau_k\succ\tau'_k)$. In round $t$, the preference dataset is $\cD_t = \{(x_k, y_k)\}_{k=1}^{t-1}$ and the corresponding (constrained) maximum likelihood estimator (MLE) is given by $\thetahat_t = \argmin_{\norm{\theta}\leq B} \cL_{\cD_t}(\theta)$, where
\begin{equation}\label{eq:mle}
    \cL_{\cD_t}(\theta):= - \sum_{(x,y) \in \cD_t} \bs{y \log\sigma(\ip{\theta}{x}) + (1-y)\log\sigma(-\ip{\theta}{x})},
\end{equation}
is the negative log-likelihood of the Bradley-Terry model \eqref{eq:bt_model}. The loss function \eqref{eq:mle} is the familiar logistic loss from logistic regression \citep{shalev2014understanding}. In particular, it is a convex problem that can be solved efficiently using standard methods such as LBFGS \citep{liu1989limited}. Moreover, we have the following time-uniform confidence result.
\begin{restatable}{lemma}{confidence}\label{lem:confidence_set}
    Let $\lambda\geq 0$ and define the design matrix at time $t$ given by $V_t = \lambda I + \sum_{k=1}^{t-1}x_k x_k^\top$. Then, with probability $1-\delta$, for all $t\in\bN$, the true reward parameter $\theta^*$ is contained in the ellipsoid
    \begin{equation*}
        \cE_t(\delta) := \bc{\theta: \norm{\theta - \thetahat_t}_{V_t}^2 \leq \beta_t(\delta)^2:= \cO\br{\kappa\bs{\log\br{\frac{1}{\delta}} + d \log \br{\frac{t-1}{d}}} + \lambda}}.
    \end{equation*}
    Here, $\kappa:=\max_{\theta\in\cB^d(B), x \in \cB^d(2L\effH)}1/\dot{\sigma}\br{\ip{\theta}{x}}$ denotes the Lipschitz constant of the inverse sigmoid function, and $\effH = (1-\gamma)^{-1}$ the effective horizon of the MDP.
\end{restatable}
The above lemma hinges on a result for likelihood-ratio confidence sets by \citet{lee2024unified}. The proof and the precise constants are deferred to Appendix~\ref{app:sec:technical_results}. 

\begin{remark}\label{rem:kappa}
Compared to the standard analysis of stochastic linear bandits \citet{abbasi2011improved}, our parameter $\beta_t$ includes an additional factor of $\sqrt{\kappa}$, which arises naturally due to preference-based feedback. This result improves upon the bound provided by \citet{zhu2023principled}, which incurs a larger factor of $\kappa$ instead of $\sqrt{\kappa}$. While the $\sqrt{\kappa}$ factor can theoretically be avoided by constructing confidence sets using the Hessian of the negative log-likelihood $\cL_{\cD_t}$ \citep{lee2024unified, das2024active}, it reappears in the regret bounds as shown in \citep{das2024active}. We adopt confidence sets based on $V_t$, as this facilitates preference-free exploration, and both $V_t$ and its inverse can be efficiently updated via rank-one operations, unlike Hessian-based approaches.
\end{remark}
\vspace{-0.1cm}

\paragraph{Randomized exploration} Many approaches to regret minimization and pure exploration in RLHF rely on maximizing an exploration bonus of the form $\norm{\phi(\pi_t)-\phi(\pi_t')}_{V_t^{-1}}$ or {\small$\E_{\tau\sim\bP_{\pi_t}, \tau'\sim \bP_{\pi_t'}}\norm{\phi(\tau)-\phi(\tau')}_{V_t^{-1}}^2$}. Although such methods yield provable guarantees for regret \citep{pacchiano2021dueling} or last-iterate suboptimality \citep{das2024active}, they are computationally intractable in RL settings (see Appendix~\ref{app:sec:opt_approaches}). To address this, we adopt a randomized exploration scheme inspired by Thompson sampling algorithms for linear bandits. In line with \citet{pmlr-v54-abeille17a}, we sample the reward parameter from an inflated version of the confidence set defined in Lemma \ref{lem:confidence_set}, which produces a computationally efficient alternative to optimism-based approaches. Furthermore, we also show that this randomized strategy extends to the pure exploration setting.
\vspace{-0.1cm}

\paragraph{RL oracle} With the objective of a meta-algorithm, we assume access to the following RL oracle. 
\begin{assumption}[PAC-RL oracle]
We assume access to an $(\varepsilon, \delta)$-PAC oracle, $\oracle$, for the RL problem. That is, a polynomial-time algorithm that produces for every $\varepsilon>0$, $\delta\in (0,1)$, and $\theta\in\R^d$ a policy $\pi = \oracle(\theta, \varepsilon, \delta)$ such that with probability at least $1-\delta$ we have $V^*_{\theta} - V_{\theta}^{\pi}\leq \varepsilon$.
\end{assumption}
This assumption of a PAC-RL oracle is satisfied in several settings, including tabular and linear MDPs \citep{dann2019policy,menard2021fast,al2021navigating,wagenmaker2022beyond, he2021uniform}. Moreover, it is relatively mild compared to stronger oracles considered in the RLHF literature \citep{zhan2024provablereward}, such as reward-free algorithms \citep{wang2020reward, kaufmann2021adaptive, menard2021fast}. In practice, common choices for $\oracle$ are policy optimization methods such as proximal policy optimization (PPO) \citep{schulman2017proximal} or soft actor critic \citep{haarnoja2018soft}, which have shown strong empirical performance in continuous control and large-scale applications such as training large language models.
\vspace{-0.1cm}

\paragraph{Algorithm statement} Our algorithm randomized preference optimization (\texttt{RPO}) presented below comes in two variants: (i) \texttt{RPO-Regret} (Algorithm~\ref{alg:regret}) which balances exploration and exploitation for regret minimization, and (ii) \texttt{RPO-Explore} (Algorithm~\ref{alg:explore}) which performs preference-free exploration and collects a single batch of preferences at the end. In \texttt{RPO-Regret}, we sample in each round a reward parameter from a confidence set (see Lemma~\ref{lem:confidence_set}) inflated by {\small$\sqrt{d}$}, compute the policy $\pi_t$ via the RL oracle, and update the reward estimate by maximum likelihood using the newly collected preference. In contrast, \texttt{RPO-Explore} samples reward parameters centered at zero, stores trajectory pairs during exploration, and collects a single batch of preferences at the end.
\vspace{-0.1cm}

\begin{algorithm}[H]
\DontPrintSemicolon
\label{alg:regret}
    \SetAlgoLined
    \KwIn{$T\in \bN,\ \delta\in (0,1),\ \lambda>0$}
    \textbf{Initialize:} $\varepsilon = 1/\sqrt{T}$;\ $\delta'= \delta/5$;\ $V_1= \lambda I$;\ $\cD_1=\emptyset$;\ $\hat\theta_1=0$;\ and choose $\pi_0$.
    
    \For{$t=1,2,\dots,T$}{
        $\thetatilde_t \sim \cN\big(\thetahat_t,\beta_t(\delta')^2 V_t^{-1}\big)$ 
        \tcp*[r]{Reward sampling}
        $\pi_t= \oracle\big(\thetatilde_t, \varepsilon, \delta'/T\big), \pi'_t=\pi_{t-1}$\tcp*[r]{RL update}
        $\tau_t\sim \bP_{\pi_t}, \tau'_t\sim \bP_{\pi'_t}$\\
        $x_t = \phi(\tau_t) - \phi(\tau_t')$, $V_{t+1} = V_t + x_t x_t^\top$ \\
        $y_t = \ind(\tau_t \succ \tau'_t)$, $\cD_{t+1} = \cD_t \cup \bc{\br{x_t, y_t}}$ \tcp*[r]{Preference feedback}
        $\hat\theta_{t+1} \in \argmin_{\norm{\theta}\leq B} \cL_{\cD_{t+1}}(\theta)$ \tcp*[r]{Reward estimation}
    }
    \caption{\texttt{RPO–Regret} (online preference learning for regret minimization)}
\end{algorithm}
\vspace{-0.1cm}

\begin{algorithm}[H]
\DontPrintSemicolon
\label{alg:explore}
    \SetAlgoLined
    \KwIn{$T\in \bN,\ \delta\in (0,1),\ \lambda>0$}
    \textbf{Initialize:} $\varepsilon = 1/\sqrt{T}$;\ $\delta'= \delta/5$;\ $V_1= \lambda I$;\ $\cB = \cD =\emptyset$;\ and choose $\pi_0$.

    \For{$t=1,2,\dots,T$}{
    $\thetatilde_t \sim \cN\big(0,V_t^{-1}\big)$ 
    \tcp*[r]{Reward sampling}
    $\pi_t= \oracle\big(\thetatilde_t, \varepsilon, \delta'/T\big), \ \pi'_t=\pi_{t-1}$ \tcp*[r]{RL update}
    $\tau_t\sim \bP_{\pi_t}, \ \tau'_t\sim \bP_{\pi'_t}$\\
    $x_t = \phi(\tau_t) - \phi(\tau_t')$, $V_{t+1} = V_t + x_t x_t^\top$\\
    $\cB = \cB \cup \{(\tau_t,\tau'_t,x_t)\}$ \tcp*[r]{Defer preference feedback}
    }
    \ForEach{$(\tau,\tau',x)\in\cB$}{
        $y = \ind(\tau \succ \tau')$;\ $\cD = \cD \cup \{(x, y)\}$\tcp*[r]{Preference feedback}
    }
    \KwOut{Policy $\pihat = \oracle\br{\thetahat, \varepsilon, \delta'}$ with $\hat\theta \in \argmin_{\norm{\theta}\leq B} \cL_{\cD}(\theta)$ \tcp*[r]{One MLE at end}}
    \caption{\texttt{RPO–Explore} (preference-free exploration and batched reward estimation)}
\end{algorithm}
\vspace{-0.1cm}

\subsection{Theoretical results} 
We analyze the regret of \hyperref[alg:regret]{\texttt{RPO-Regret}} and the suboptimality of the output policy $\pihat$ of \hyperref[alg:explore]{\texttt{RPO-Explore}}. 

\subsubsection{Regret analysis} 

We show that \texttt{RPO-Regret} incurs sublinear regret with high probability.
\begin{restatable}{theorem}{regret}
\label{thm:regret_thompson_sampling}
    For any $\delta\in (0,1)$ and $T\in \bN$, Algorithm \ref{alg:regret} satisfies, with probability at least $1-\delta$,
    \begin{equation*}
        R(T) = \bigO\left(\sqrt{\kappa d^3 T \log(dT/\delta)^3} \right).
    \end{equation*}
\end{restatable}
The regret bound of Theorem \ref{thm:regret_thompson_sampling} matches the best existing bounds for algorithms with randomized exploration in reinforcement learning, see \citep{efroni2021reinforcement,ouhamma2023bilinear}. In addition, due to learning from preferences, we have an extra $\sqrt{\kappa}$ factor (see Remark~\ref{rem:kappa}), which is in line with other recent work on RLHF \citep{wu2023making, das2024active}. 

\paragraph{Comparison with prior work} For episodic tabular MDPs \citet{pacchiano2021dueling} prove a regret bound of $\bigOt(\kappa d\sqrt{T})$. Similarly, \citet{chen2022human} considers episodic linear MDPs and derives a regret bound of $\bigOt(d\sqrt{HT})$, avoiding dependence on $\kappa$ by assuming a linear preference model. However, these approaches rely on a type of optimism which is computationally intractable in this setting (see Appendix~\ref{app:sec:opt_approaches}). Similar to us, \citet{wu2023making} avoid this challenge by resorting to randomized exploration and proves a $\bigOt(d^3\sqrt{\kappa T})$ regret bound for linear MDPs. Compared to \citet{wu2023making}, our analysis improves the dependence on dimensionality from $d^3$ to $d^{3/2}$ and avoids the need for truncation techniques on the value function. Furthermore, our settings differ in two key points: First, we assume access to an RL oracle without restricting the class of MDPs, whereas \citet{wu2023making} considers linear MDPs. Second, their approach is model-based, while Algorithm \ref{alg:regret} is oracle-based and can accommodate both model-based and model-free implementations.

\paragraph{Proof idea} The full proof of Theorem~\ref{thm:regret_thompson_sampling} is provided in Appendix~\ref{app:proof_regret_thompson_sampling}. Analogous to the analysis of linear Thompson sampling \citep{pmlr-v54-abeille17a}, the main idea is to control the regret by showing that randomized exploration ensures a constant probability of optimism. However, compared to the linear bandit analysis, our setting comes with additional challenges: First, due to preference-based learning we require a different regret decomposition accounting for the reference policy. Second, as we observe preference feedback on trajectories rather than policies, we need to apply Freedman's inequality (see Lemma \ref{lem:freedman}) to control the deviation between expected and observed features. Lastly, as we assume a PAC RL oracle -- in place of an exact maximization oracle -- we need to carefully track the resulting approximation error.

\subsubsection{Suboptimality gap} 
As \hyperref[alg:explore]{\texttt{RPO-Explore}} collects no preferences during exploration, it may incur linear regret as the policies $\pi_t$ can be highly suboptimal. However, Theorem~\ref{thm:last_iterate} below shows that the final output policy is $\tilde\cO(1/\sqrt{T})$-optimal.
\begin{restatable}{theorem}{lastiterate}\label{thm:last_iterate}
    For any $\delta\in (0,1)$ and $T\in \bN$, Algorithm \ref{alg:explore} satisfies, with probability at least $1-\delta$,
\begin{equation*}
    \subopt\br{\pihat} = \cO\br{\sqrt{\frac{\kappa d^3}{T}\log\br{\frac{dT}{\delta}}^3}}.
\end{equation*}
In other words, we need $\tilde\cO(\kappa d^3 / \varepsilon^2)$ iterations to output an $\varepsilon$-optimal policy with high probability.
\end{restatable}
\vspace{-0.1cm}

Except for the extra {\small$\sqrt{d}$} dependency, which is inherent to approaches based on randomized exploration,\footnote{Recently, \citet{abeille2025and} showed that, in certain cases, the factor {\small$\sqrt{d}$} can be removed by an improved analysis. Their assumptions do not hold in our setting, so whether {\small$\sqrt{d}$} can be avoided remains open.} we match the last iterate guarantee proposed by \citet{das2024active} for a contextual linear bandit setting, but with an algorithm that (i) collects a single batch of preferences and (ii) remains tractable in a full RL setting with trajectory-level feedback, given the PAC RL oracle is tractable.
\vspace{-0.1cm}

\paragraph{Comparison with prior work} Algorithm~\ref{alg:explore} is reminiscent of reward-free RL algorithms \citep{wang2020reward}, but in a preference-based setting. To our knowledge, Algorithm~\ref{alg:explore} is the first tractable algorithm to perform efficient preference-free exploration with trajectory-level preferences, and the first -- across preference- or reward-based RL -- to do so via randomized exploration. Few other works provide suboptimality guarantees in preference-based RL, but require preference-feedback at every round. \citet{wang2023rlhf} propose a meta-algorithm interfacing with a PAC-RL oracle that outputs an $\varepsilon$-optimal policy after $\bigOt(\kappa^2 d^3 / \varepsilon^2)$ queries. A different approach by \citet{zhan2024provablereward} leverages optimal design to prove a bound of $\bigOt((|S|^2 |A| d+\kappa^2 d^2)/\varepsilon^2)$, but their method relies on an intractable maximization oracle. In comparison, Algorithm~\ref{alg:explore} achieves a bound of $\bigOt(\kappa d^3 / \varepsilon^2)$, improving the dependence on $\kappa$ over both prior results. The extra {\small$\sqrt{d}$} compared to \citet{zhan2024provablereward} is expected for randomized (rather than optimistic) exploration \citep{pmlr-v54-abeille17a}. Finally, note that an online-to-batch conversion yields similar suboptimality bounds for Algorithm~\ref{alg:regret}, but would require preferences at every round \citep{menard2021fast}.
\vspace{-0.1cm}

\paragraph{Proof idea} The proof of Theorem~\ref{thm:last_iterate}, presented in Appendix \ref{app:proof_last_iterate}, builds on \citet{das2024active}'s suboptimality analysis for the contextual bandits setting. However, to sidestep the intractability of maximizing an exploration bonus over policies, we leverage randomized exploration \citep{pmlr-v54-abeille17a} to ensure a constant probability of optimism. This allows us to derive a bound on the output policy’s suboptimality that mirrors the regret bound, without needing additional assumptions.
\vspace{-0.1cm}

\subsection{Practical limitations} While Algorithm \ref{alg:regret} is tractable and efficient, it presents certain limitations. Issuing preference queries at each round (line 7) is impractical, due to the need for continuous feedback, and requesting a label for all trajectory pairs can be expensive and inefficient, as many comparisons may be uninformative. Moreover, the large regret of Algorithm~\ref{alg:explore} may be undesirable for certain applications. The next section introduces a refined regret-minimization algorithm that addresses these issues by decoupling trajectory collection from query selection and querying only the most informative comparisons.

\vspace{-0.1cm}
\section{A practical algorithm with efficient query selection}
\label{sec:practical_algorithm}
\vspace{-0.1cm}

We present Algorithm \ref{alg:lazy_thompson_sampling}, an improved method for preference collection and active query selection.
\vspace{-0.1cm}

\subsection{Algorithm}
\vspace{-0.1cm}

As discussed earlier, we design Algorithm \ref{alg:lazy_thompson_sampling} by using lazy updates to collect a batch of trajectory pairs, then applying optimal design to select the informative queries from the batch.
\begin{algorithm}[H]
\label{alg:lazy_thompson_sampling}
\DontPrintSemicolon
    \SetAlgoLined
    \KwIn{$T\in\bN$, $\delta\in (0,1)$, $\lambda>1$, $C>0$}
    
    \textbf{Initialize:} $\varepsilon = 1/\sqrt{T}$;\ $\delta'= \delta/5$;\ $V_1= W_1 = \lambda I$;\ $\cD_1 = \cD =\emptyset$;\ $\ts=1$ ;\ $\thetahat_1=0\, \text{;}\; \pi'$.
    
    \For{$t=1,2,\dots,T$}{
        \If{$\det(W_t) > (1+C) \det(V_{\ts})$}{
        $\cD_{\text{opt}}, V_t = \operatorname{D-OptDes}(\cD, V_{\ts}, W_t)$ \tcp*[r]{Optimal design}
        $\cD_t = \cD_{\ts}$\\
        \For{$(\tau, \tau') \in \cD_{\text{opt}}$}{
        $\cD_t = \cD_{t} \cup \{(x, y)\},  x = \phi(\tau)-\phi(\tau'), y = \ind(\tau \succ \tau')$ \tcp*[r]{Preferences}
        }
        $\hat\theta_{t} \in \argmin_{\norm{\theta}\leq B} \cL_{\cD_{t}}(\theta)$ \tcp*[r]{Reward estimation}
        $\ts=t$, $\cD = \emptyset$, $\pi' = \oracle\br{\thetahat_t, \varepsilon, \delta'/T}$ \\
        }
        $\thetatilde_t \sim \cN(\thetahat_{\ts},\beta_{\ts}(\delta')^2 V_{\ts}^{-1})$\tcp*[r]{Reward sampling}
        $\pi_t= \oracle\br{\thetatilde_t, \varepsilon, \delta'/T}$\tcp*[r]{Update policy with RL}
        $\cD = \cD \cup \{(\tau_t, \tau'_t)\}$, $x_t = \phi(\tau_t) - \phi(\tau_t')$ with $\tau_t\sim \bP_{\pi_t}, \tau'_t\sim \bP_{\pi'}$\\
        If $t = \ts$, then $W_{t+1} = V_t + x_t x_t^\top$, else: $W_{t+1} = W_t + x_t x_t^\top$\\
    }
    \caption{\lalgo~ (lazy randomized preference optimization with optimal design)}
\end{algorithm}
\vspace{-0.1cm}
\paragraph{Lazy updates} We use an idea from \citet{abbasi2011improved} to collect many trajectory pairs without querying the preference oracle. The modification compared to Algorithm \ref{alg:regret} is collecting trajectories without updating the MLE $\hat{\theta}_{\ts}$ until the information gain, represented by $\det(V_t)$, increases by a multiplicative constant; see line 4 of Algorithm \ref{alg:lazy_thompson_sampling}. We show that this procedure limits the number of batches to $\bigO(\log(T))$. In other words, the average (over batches) size of a given batch is of order $\bigO(T/\log(T))$. A key advantage of this lazy update structure is that the preference queries (line 7) can be collected in parallel across all trajectory pairs within a batch. This significantly reduces the workload of the preference oracle, \eg a human annotator, by eliminating the need for round-by-round feedback, and the algorithm no longer pauses at each timestep to wait for preference labels.\looseness-1

\paragraph{D-Optimal design} To select informative preference queries from the collected trajectories above, we leverage tools from optimal experimental design. Specifically, we apply an approximate D-optimal design criterion to each collected batch of trajectory pairs; see Appendix \ref{ap:background_optimal_design} for background on D-optimal design. Given the current matrix $V_{\ts}$ and the set of candidate trajectory pairs $\cD$, we use a greedy algorithm to solve the following maximization problem:
\begin{equation*}
    \max_{\{n_x\}} \log\det\br{V_{\ts}+\sum_{x \in \cD} n_x x x^\top} \quad \text{ subject to } \sum_{x \in \cD} n_x = |\cD|, \:\: n_x \in \bN.
\end{equation*}
Due to the submodularity of the $\log\det$ function for $\lambda$ greater than one, the greedy procedure of Algorithm \ref{alg:greedy_optimal_design} achieves an $(1-1/e)$-approximation to the optimal solution; see \citep{nemhauser1978analysis,krause2008near} and Appendix \ref{ap:background_optimal_design}.

Another key feature of Algorithm \ref{alg:greedy_optimal_design} is its early stopping rule: the while loop terminates early if $\det(V)$ exceeds the determinant of the naive design $W$; if this never happens we simply return $W$. When this early termination is satisfied, Algorithm \ref{alg:lazy_thompson_sampling} requires fewer preferences than Algorithm \ref{alg:regret}. Since the optimal design maximizes the information gain (as measured by $\det(V)$), this termination condition is expected to be satisfied frequently in practice.

\begin{algorithm}[H]
    \label{alg:greedy_optimal_design}
    \SetAlgoLined
    \KwIn{Dataset $\cD$, previous design matrix $V$, current design matrix $W$ including $\cD$.}
    Initialize dataset $\cD_{\text{opt}} = \emptyset$\\
    \While{$\det(V) < \det(W)$ and $|\cD_{\text{opt}} | \le |\cD|$}{
    $(\tau, \tau') = \argmax_{(\tau, \tau') \in \cD} \det(V+(\phi(\tau)-\phi(\tau'))(\phi(\tau)-\phi(\tau'))^\top)$ \\
    $V = V + x x^\top, \text{ where } x = \phi(\tau)-\phi(\tau'); \: \cD_{\text{opt}} = \cD_{\text{opt}} \cup \{(\tau, \tau')\}$ \\
    }
        \KwOut{$(\cD_{\mathrm{opt}}, V)$ if $\det(V) \geq \det(W)$, else $(\cD, W)$}
    \caption{Greedy D-Optimal Design}
\end{algorithm}

\subsection{Theoretical result}

We now provide our high probability regret bound for Algorithm \ref{alg:lazy_thompson_sampling}.

\begin{restatable}{theorem}{lazy}
    \label{thm:regret_lazy_thompson_sampling}
    Instantiating Algorithm~\ref{alg:lazy_thompson_sampling} with $C>0$ and $\lambda\geq 4 \effH^2 L^2$, it holds for any $\delta\in (0,1)$ and $T\in \bN$, with probability at least $1-\delta$ that
    \begin{equation*}
        R(T) \le \bigO\left(\sqrt{(1+C) \kappa d^3 T \log(d T/\delta)^3} \right).
    \end{equation*}
    In addition, the number of times the condition of line 4 holds is at most $\frac{d}{\log(1+C)}\log\br{1 + \frac{T(2L\effH)^2}{d\lambda}}$. Therefore, the size of the batches is on average of order $\tilde{\bigO}(T/(d\log(T)))$.
\end{restatable}

Compared to Theorem \ref{thm:regret_thompson_sampling}, the above regret bound increases only by constant factors. Regarding the number of preference queries, \citet{sekhari2023contextual} shows that at least $\Omega(T)$ queries are required to achieve $\bigO(\sqrt{T})$ regret in the worst case. However, optimal design may lead to significantly fewer queries in favorable instances, as demonstrated in our experiments, while preserving the worst-case guarantees.

\paragraph{Proof idea} 
We briefly outline the main idea of the proof of Theorem~\ref{thm:regret_lazy_thompson_sampling}; the full argument is deferred to Appendix~\ref{app:proof_regret_lazy_thompson_sampling}. The central observation is that similarly as for the standard lazy update analysis \citep{abbasi2011improved}, the regret only increases by a constant factor $(1+C)$ given that the optimal design subroutine ensures that $\det V_{\ts} \geq \det W_{\ts}$ and $|\cD_{\text{opt}}| \leq |\cD|$.

\section{Experiments}
\label{sec:experiments}
We first validate our theoretical results on regret minimization in a tabular gridworld environment, where our RL oracle assumption provably holds, and then compare Algorithm~\ref{alg:regret} and \ref{alg:lazy_thompson_sampling} on more challenging continuous control tasks. The validation of Algorithm~\ref{alg:explore} is deferred to Appendix~\ref{app:sec:experiments}.

\subsection{Tabular environment}
\begin{wrapfigure}{r}{0.48\textwidth}
  \centering
  \includegraphics[width=0.48\textwidth]{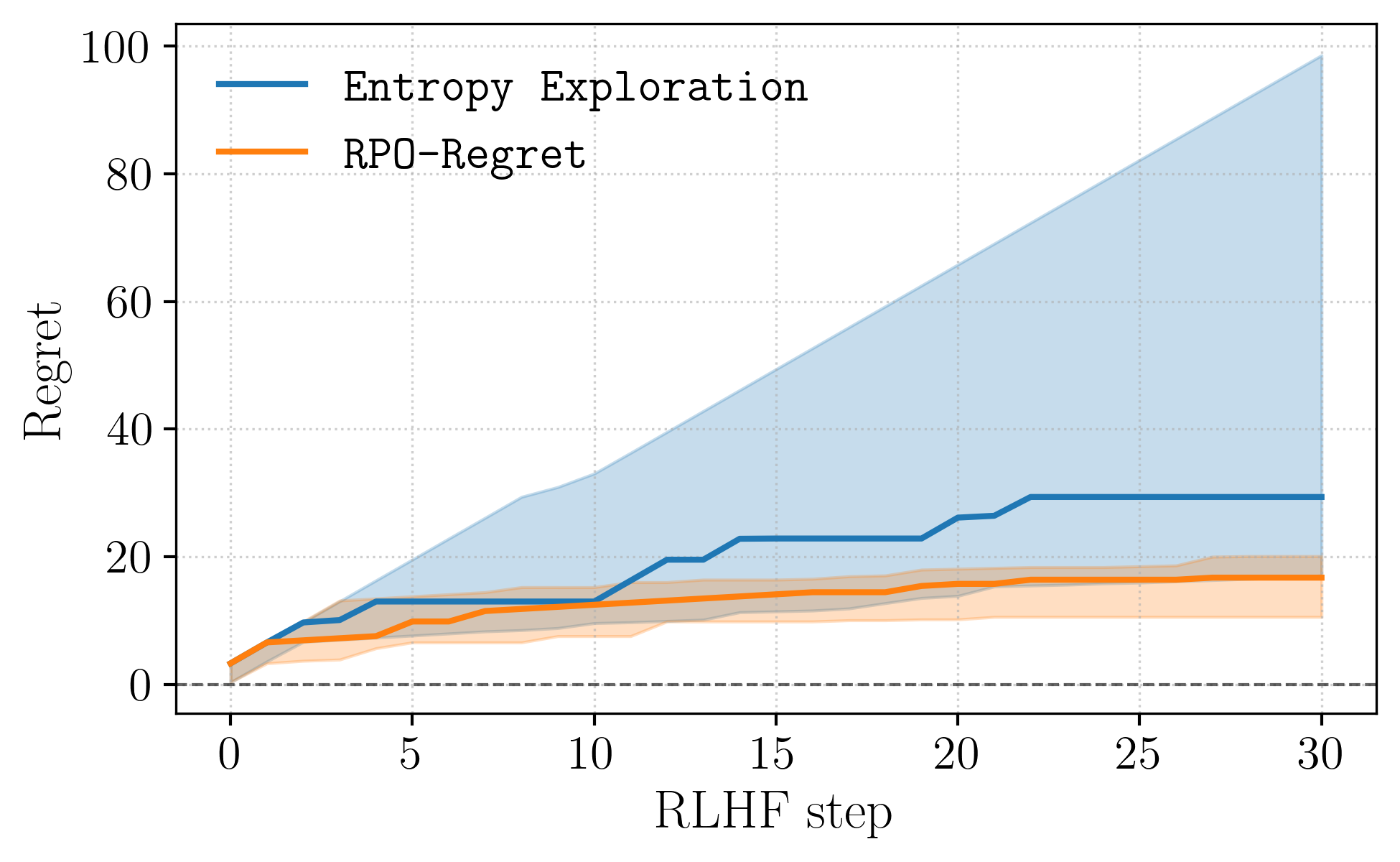}
    \caption{We compare the regret of \hyperref[alg:regret]{\texttt{RPO-Regret}} (orange, Algorithm~\ref{alg:regret}) against a baseline with entropy exploration (blue). The solid lines indicate the median and the shaded areas the 0.2 and 0.8 quantiles, across 20 independent runs. The regret is computed with respect to the ground truth reward parameter $\theta^*$.}\label{fig:gridworld_regret}
\end{wrapfigure}

We consider a $6\times 6$ gridworld environment with deterministic transitions and $4$ actions (up, down, left, right). The agent starts in the center and receives reward $0.5$ if it reaches one of two boundary states. The reward features are one-hot features of six boundary states -- including the two goal states (see Figure~\ref{fig:gridworld_environment} for an illustration of the environment). We compare \hyperref[alg:regret]{\texttt{RPO-Regret}} (Algorithm~\ref{alg:regret}) against an RLHF baseline that explores purely through entropy regularization, using synthetic preferences generated from the ground truth reward parameter $\theta^*$. We compute optimal policies using soft value iteration \citep{haarnoja2017reinforcement}, and to reduce variance in the reward estimate, we sample $100$ trajectories from $\pi_t$ and $\pi_t'$ in each RLHF round. 

As shown in Figure~\ref{fig:gridworld_regret}, \hyperref[alg:regret]{\texttt{RPO-Regret}} attains considerably lower regret with less variance across runs than the baseline. These results underscore that entropy exploration, does not necessarily guarantee low regret in MDPs.

\subsection{Continuous control}
\begin{figure}[!b]
\centering
\begin{subfigure}[t]{0.48\textwidth}
    \centering
    \includegraphics[width=\textwidth]{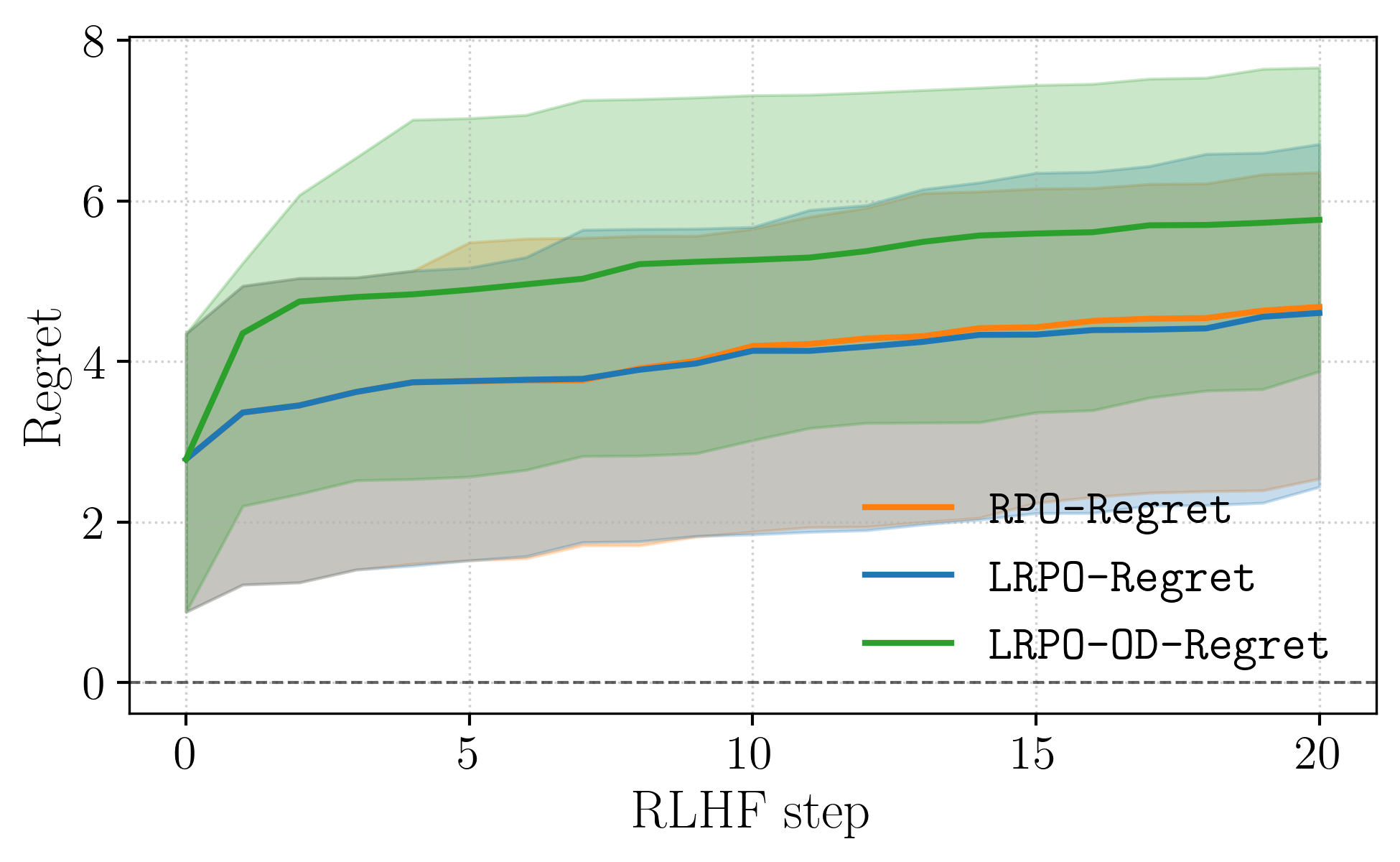}
    \label{fig:ground_truth_reward}
    
    \vspace{-0.5cm}
    \hspace{0.5cm}(a)
\end{subfigure}
\hfill
\begin{subfigure}[t]{0.48\textwidth}
    \includegraphics[width=\textwidth]{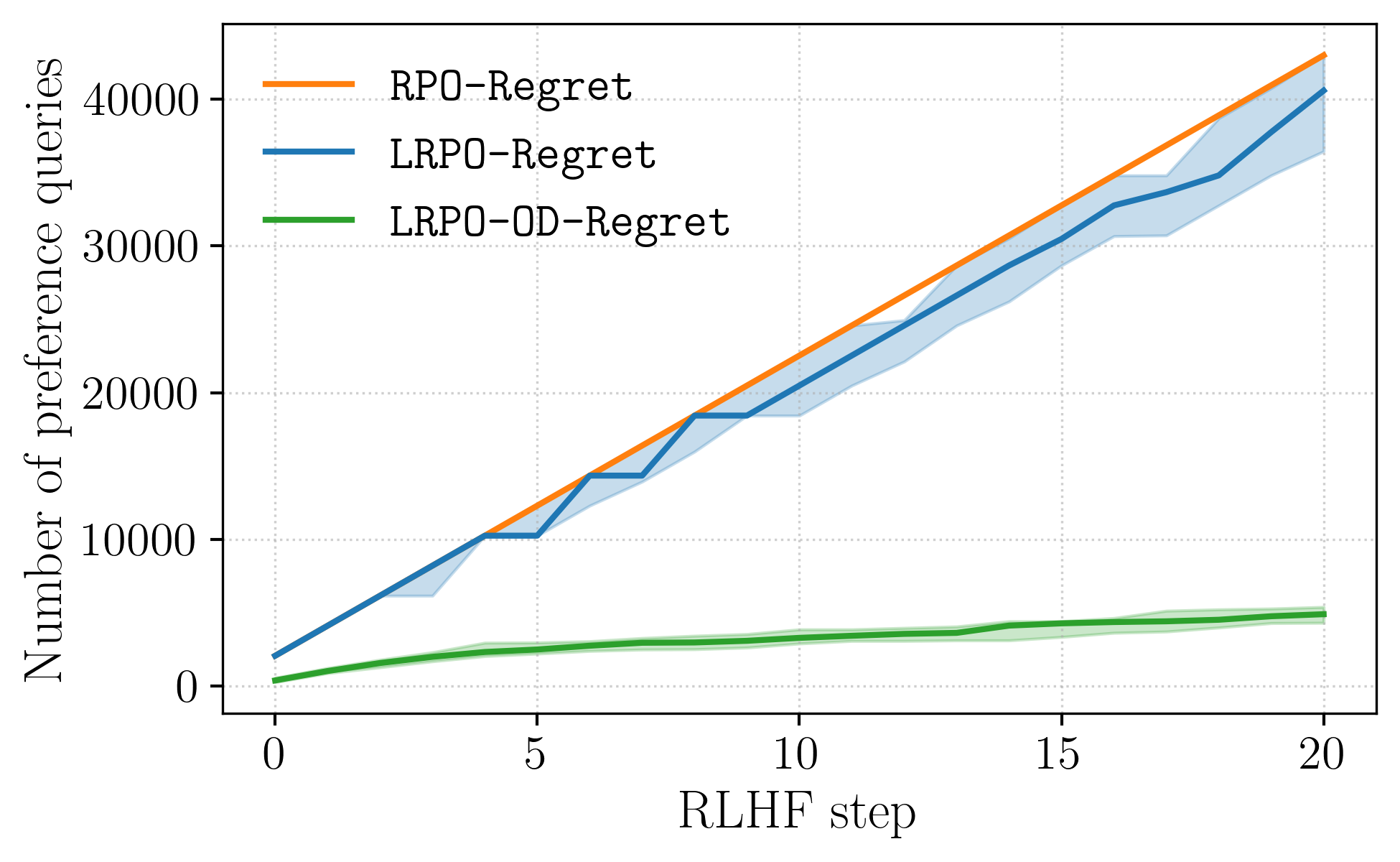}
    \label{fig:num_queries}
    
    \vspace{-0.5cm}
    \hspace{3.3cm}(b)
\end{subfigure}
\caption{
Comparison of regret minimization algorithms in terms of (a) the cumulative regret (estimated from samples against an RL policy trained with $\theta^*$) and (b) number of preference queries performed. In particular, we compare \hyperref[alg:regret]{\texttt{RPO-Regret}} (orange, Algorithm~\ref{alg:regret}) with its lazy versions \hyperref[alg:lazy_thompson_sampling]{\texttt{LRPO-Regret}} and \hyperref[alg:lazy_thompson_sampling]{\lalgo}~(blue \& green). Here, \hyperref[alg:lazy_thompson_sampling]{\texttt{LRPO-Regret}} and \hyperref[alg:lazy_thompson_sampling]{\lalgo}~refer to Algorithm~\ref{alg:lazy_thompson_sampling} without and with optimal design subroutine. The solid lines indicate the median and the shaded areas the 0.2 and 0.8 quantiles, across 20 independent runs. }
\label{fig:experiments}
\end{figure}
We validate our theoretical results on regret minimization (Theorems~\ref{thm:regret_thompson_sampling} and \ref{thm:regret_lazy_thompson_sampling}) on the \texttt{Isaac-Cartpole-v0} environment from Nvidia Isaac Lab \citep{mittal2023orbit}. In this task the goal is to balance a pole on a cart by applying left or right forces, preventing the pole from falling. We compare \hyperref[alg:regret]{\texttt{RPO-Regret}} with its lazy variants (\hyperref[alg:lazy_thompson_sampling]{\texttt{LRPO-Regret}}~without optimal design and \hyperref[alg:lazy_thompson_sampling]{\lalgo}~with optimal design). To simulate human preferences, we generate synthetic preferences using the built-in task-specific reward function, which is a linear combination of $5$ reward terms (see Appendix~\ref{app:sec:experiments}). Furthermore, we adopt PPO \citep{schulman2017proximal} as our RL oracle using 30 PPO steps per iteration of Algorithm~\ref{alg:regret} and \ref{alg:lazy_thompson_sampling}. Again, we sample $100$ trajectories from $\pi_t$ and $\pi_t'$ in each RLHF round.

Figure~\ref{fig:experiments} shows that for all three algorithms the regret slope flattens after a few RLHF rounds, demonstrating sublinear regret and performance competitive with RL using ground truth rewards. Although \lalgo~suffers slightly higher regret initially, it quickly reaches the same performance as \texttt{RPO-Regret} and \texttt{LRPO-Regret}, while reducing the number of preference queries considerably. This highlights that, despite theoretical worst-case lower bounds, the number of preference queries can be considerably reduced in practice by selecting informative queries with optimal design. Additional experimental results and further implementation and evaluation details are provided in Appendix~\ref{app:sec:experiments}. 

\section{Conclusion}
\label{sec:conclusion}
We presented two simple meta-algorithms for reinforcement learning from human feedback (RLHF) that combine an RL oracle and randomized exploration, achieving complementary guarantees: regret bounds for Algorithm~\ref{alg:regret} and PAC-style guarantees for Algorithm~\ref{alg:explore}. Algorithm~\ref{alg:explore}, to our knowledge, is the first tractable method to perform preference-free exploration with trajectory-level preferences. Building on this framework, we introduced Algorithm~\ref{alg:lazy_thompson_sampling}, a practical regret-minimization algorithm that combines lazy updates to enable parallelization with optimal design to reduce query complexity, while maintaining regret guarantees. Empirically, our approach is competitive with reward-based RL while requiring significantly fewer preference queries. Overall, our contributions advance the state of RLHF by combining strong theoretical guarantees with practical algorithm design, improving efficiency and broadening applicability to real-world scenarios.


Our work opens several directions for future research. First, we adopt the Bradley-Terry model for preference generation, which may not fully capture the complexity of real-world human feedback. Extending the framework to richer preference models is an important direction. Second, our approach relies on an RL oracle at every RLHF step, which may be computationally demanding. While using a reward-free algorithm as an RL oracle is theoretically efficient, practical RL implementations are typically based on policy optimization, which is not a reward-free algorithm, and often entails high sample complexity. Thus, it remains open whether we could require fewer RL oracle calls or whether reward-free oracles can be successfully implemented. Finally, our experimental evaluation is limited to tabular settings and simple robotic control tasks with synthetic feedback. Assessing performance on more complex tasks and with real human or LLM-generated feedback would offer a stronger test of the method's practical applicability.

\paragraph{Acknowledgments}
Andreas Schlaginhaufen is funded by a PhD fellowship from
the Swiss Data Science Center.

\bibliographystyle{plainnat}
\bibliography{main}

\newpage

\newpage
\appendix
\addcontentsline{toc}{section}{Appendix} 
\part{Appendix} 
\parttoc 
\newpage

\section{Proof of Theorem \ref{thm:regret_thompson_sampling}}
\label{app:proof_regret_thompson_sampling}
\regret*
\begin{proof}
    The proof proceeds in four steps. We first define a set of good events of concentration of parameters and sums of trajectory features, we show that they are satisfied with high probability. Second, we decompose the regret into two terms: a pessimism term and an estimation error term. We then show that the pessimism term is controlled by establishing a constant probability of optimism, and the estimation error term is controlled by the estimation error of the maximum likelihood estimator.
    
    Throughout the proof, we work with the two filtrations $\cF_{t-1}=\sigma(x_1,y_1,\hdots, x_{t-1}, y_{t-1})$ and $\cF_{t-1}^\theta=\sigma(\cF_{t-1}, \thetatilde_t)$. In particular, both $\thetahat_t$ and $V_t$ are $\cF_{t-1}$ measurable. Therefore, $\thetatilde_t$ follows a Gaussian distribution given $\cF_{t-1}$, \ie $\thetatilde_t \mid \cF_{t-1} \sim \cN(\thetahat_t, \beta_t^2 V_t^{-1})$, while it is fully determined by $\cF_{t-1}^{\theta}$.

\textit{Step 1 (good events)}: Recall that in Algorithm~\ref{alg:regret} we set $\varepsilon=1/\sqrt{T}$ and $\delta'=\delta/5$. We define the following high probability events. 
\begin{enumerate}
    \item Let $\delta'' = \delta'/T$ and $c(\delta''):= \sqrt{2d\log(2d/\delta'')} = \sqrt{2d\log(10dT/\delta)}$. Consider the inflated ellipsoid
        \begin{equation*}
            \cE_t^{\text{TS}} := \bc{\theta\in\R^d: \norm{\theta - \hat\theta_t}_{V_t}\leq \beta_t(\delta')c(\delta'')} = c(\delta'')\cE_t(\delta').
        \end{equation*}
        We define the events
        \begin{equation*}
            \widehat E_t:=\bc{\theta^* \in \cE_t(\delta')},\quad \widetilde E_t:=
                \bc{\tilde \theta_t\in \cE_t^{\text{TS}}}, \quad \text{ and } E_t:= \widehat E_t \cap \widetilde E_t
        \end{equation*}
        and let $G_1:= \bigcap_{t=1}^T E_t$. This event implies that $\theta^*$ lies in the confidence set uniformly over times $t\in[1, T]$, and the sampled parameter is close to $\thetahat_t$.
    \item Let $G_2$ denote the event that for
    {\small
    \begin{equation*}
        C_T:=2 \sqrt{T \br{2 d \log\left(1 \!+\! \frac{4TL^2\effH^2}{d \lambda}\right) \!\!+\! \frac{12 d L^2\effH^2}{\log(2) \lambda} \log\left(1\!+\!\frac{4L^2\effH^2}{\log(2) \lambda}\right)}}
        + \frac{16L\effH}{\sqrt{\lambda}} \log\br{\frac{2}{\delta'}},
    \end{equation*}
    }
    it holds that
    \begin{equation*}
        \sum_{t=1}^T\E\bs{\norm{\phi(\tau_t) - \phi(\tau_t')}_{V_t^{-1}} \mid \cF_{t-1}} \leq C_T,
    \end{equation*}
    and that 
    \begin{equation*}
        \sum_{t=1}^T\E\bs{\norm{\phi(\tau_t) - \phi(\tau_t')}_{V_t^{-1}} \mid \cF_{t-1}^\theta} \leq C_T.
    \end{equation*}
    \item Let $A_t = \bc{V_{\thetatilde_t}^* - V_{\thetatilde_t}^{\pi_t}\leq \varepsilon}$, where $\pi_t= \oracle\br{\thetatilde_t, \varepsilon, \delta'/T}$, and let $G_3 = \bigcap_{t=1}^T A_t$.
\end{enumerate}
As shown below, with probability at least $1-\delta$, all the above good events happen at the same time.
    
\begin{proposition}
\label{prop:good_event}\label{app:prop:good_event}
Let $G:=\bigcap_{i=1}^3 G_i$. It holds that $\Pr\bs{G} \geq 1-\delta$.
\end{proposition}
\begin{proof}
    For the event $G_1$, Lemma~\ref{lem:confidence_set} and \ref{lem:concentration_anticoncentration} imply that $\Pr\bs{\bigcup_{t=1}^T\widehat E_t^{\complement}} \leq \delta'$ and $\Pr[\widetilde E_t^{\complement}]\leq \delta'/T$ for any $t\in[1,T]$. Hence, a union bound yields $\Pr\bs{G_1^\complement}\leq 2\delta'$. Furthermore, by Lemma~\ref{lem:elliptical_lemma} 2), we have $\Pr\bs{G_2^\complement}\leq \delta'$. Moreover, by union bound and the definition of $\oracle$, we have $\Pr\bs{G_3^\complement}\leq \delta'$.  Hence, we have $\Pr\bs{G} \geq 1 - \sum_{i=1}^3 \Pr\bs{G_i^\complement}\geq 1-4\delta' \geq 1-\delta$.
\end{proof}
    
    \textit{Step 2 (regret decomposition)}: Recall that in algorithm \ref{alg:regret}, we choose $\pi_t' = \pi_{t-1}$. Hence, we can upper bound the cumulative regret as follows
    \begin{align*}
        R(T) &= \frac{1}{2}\sum_{t=1}^T \br{2V_{\theta^*}^*-V_{\theta^*}^{\pi_t}-V_{\theta^*}^{\pi'_t}}\\
        &\leq \sum_{t=1}^T \br{V_{\theta^*}^*-V_{\theta^*}^{\pi_t}} + V_{\theta^*}^*-V_{\theta^*}^{\pi_0'}\\
        &\leq \sum_{t=1}^T \underbrace{\br{V_{\theta^*}^*-V_{\theta^*}^{\pi_t}}}_{r_t} + BL\effH.
    \end{align*}
    Let $\Delta_t(\theta):= \max_{\pi}\ip{\theta}{\phi(\pi)-\phi(\pi'_t)}$. On the good event $G$, we can decompose the instantaneous regret as follows
    \begin{align*}
        r_t  &:= V_{\theta^*}^* - V_{\theta^*}^{\pi_t}\\
        &= \br{V_{\theta^*}^* - V_{\theta^*}^{\pi_t'}} - \br{V_{\thetatilde_t}^{\pi_t} - V_{\thetatilde_t}^{\pi_t'}} + \br{V_{\thetatilde_t}^{\pi_t} - V_{\thetatilde_t}^{\pi_t'}} - \br{V_{\theta^*}^{\pi_t} - V_{\theta^*}^{\pi_t'}}\\
        &\leq \br{V_{\theta^*}^* - V_{\theta^*}^{\pi_t'}} - \br{V_{\thetatilde_t}^* - V_{\thetatilde_t}^{\pi_t'}} + \varepsilon+ \br{V_{\thetatilde_t}^{\pi_t} - V_{\thetatilde_t}^{\pi_t'}} - \br{V_{\theta^*}^{\pi_t} - V_{\theta^*}^{\pi_t'}}\\
        &= \underbrace{\Delta_t(\theta^*) - \Delta_t(\tilde\theta_t)}_{r_t^{\text{TS}}} + \underbrace{\ip{\tilde\theta_t - \theta^*}{\phi(\pi_t) - \phi(\pi'_t)}}_{r_t^{\text{MLE}}} + \varepsilon.
    \end{align*} 
 Here, $r_t^{\text{TS}}$ is a pessimism term that is negative by construction for optimistic algorithms, and $r_t^{\text{MLE}}$ is related to the estimation error of the reward parameter. 

\textit{Step 3 (bounding ${r_t^{\text{TS}}}$)}: This part of the analysis highlights the distinctiveness of randomized exploration. While optimistic algorithms ensure negativity of ${r_t^{\text{TS}}}$ through their intractable optimization procedures, randomized exploration controls it using probability arguments. Specifically, following the proof of \citet{pmlr-v54-abeille17a}, we begin by bounding $r_t^{\text{TS}}$ on the good event via a conditional expectation given the optimism event. The bound on ${r_t^{\text{TS}}}$ then follows by a careful application of the anti-concentration property established in Lemma \ref{lem:anticonc}.

Conditioned on $\widetilde E_t$, we can lower bound
\begin{equation*}
    \Delta_t(\thetatilde_t)  \geq \min_{\theta \in\cE_t^{\text{TS}}} \Delta_t(\theta)= \max_{\pi}\ip{\underline{\theta}_t}{\phi(\pi)- \phi(\pi_t')} =: \underline{\Delta}_t,
\end{equation*}
for some $\underline{\theta}_t\in\cE^{\text{TS}}_t$.
Moreover, if $O_t:= \bc{\Delta_t(\thetatilde_t) \geq \Delta_t(\theta^*)}$ denotes the event of $\thetatilde_t$ being optimistic at time $t$, we can upper bound $\Delta_t(\theta^*)$ as follows
\begin{equation*}
    \Delta_t(\theta^*) \leq \E\bs{\Delta_t(\thetatilde_t) \mid \cF_{t-1}, O_t}.
\end{equation*}
Putting this together, while keeping track of the events $E_t$ and $A_t$, we have
\begin{align*}
    r_t^{\text{TS}}\ind(E_t \cap A_t) &\leq \E\bs{\br{\Delta_t(\tilde \theta_t) - \underline{\Delta}_t}\ind(E_t \cap A_t) \mid \cF_{t-1}, O_t}\\
    &\stackrel{(i)}{\leq} \E\bs{\br{\ip{\tilde \theta_t}{\phi(\pi_t)-\phi(\pi_t')} + \varepsilon - \max_{\pi}\ip{\underline{\theta}_t}{\phi(\pi)- \phi(\pi_t')}}\ind(E_t \cap A_t) \mid \cF_{t-1}, O_t}\\
    &\stackrel{(ii)}{\leq} \E\bs{\br{\ip{\tilde \theta_t}{\phi(\pi_t)-\phi(\pi_t')} - \ip{\underline{\theta}_t}{\phi(\pi_t)- \phi(\pi_t')}}\ind(E_t \cap A_t) \mid \cF_{t-1}, O_t} + \varepsilon\\
    &\stackrel{(iii)}{\leq} 2\beta_t(\delta')c(\delta'') \E\bs{\norm{\phi(\pi_t) - \phi(\pi'_t)}_{V_t^{-1}} \mid \cF_{t-1}, \widehat E_t, O_t}\Pr\bs{\widehat E_t \mid \cF_{t-1}} + \varepsilon,
\end{align*}
where $(i)$ follows from $\varepsilon$-optimality of $\pi_t$, in $(ii)$ we used that $\max_{\pi}\ip{\underline{\theta}_t}{\phi(\pi)-\phi(\pi_t)}\geq 0$, and $(iii)$ follows from the Cauchy-Schwarz inequality. By the law of total probability, we have that
\begin{align*}
\E\bs{\norm{\phi(\pi_t) - \phi(\pi'_t)}_{V_t^{-1}} \mid \cF_{t-1}, \widehat E_t, O_t} \leq \E\bs{\norm{\phi(\pi_t) - \phi(\pi'_t)}_{V_t^{-1}} \mid \cF_{t-1}, \widehat E_t} / \Pr\bs{O_t| \cF_{t-1}, \widehat E_t}.
\end{align*}

Next, as $\theta^*\in\cE_t$ on $\widehat E_t$, we have
\begin{align*}
    \Pr\bs{O_t \mid  \cF_{t-1}, \widehat E_t}
    \geq \Pr\bs{\Delta_t(\thetatilde_t) \geq \max_{\theta \in \cE_t}\Delta_t(\theta) \mid  \cF_{t-1}}.
\end{align*}
Since $\theta\mapsto\Delta_t(\theta)$ is the sum of a linear function and the function $\theta\mapsto \max_{\pi}\ip{\theta}{\phi(\pi)}$ which is convex and continuous by Proposition \ref{prop:convex}.
Then, applying Lemma~\ref{lem:anticonc} with $f(\theta)=\Delta_t(\theta)$, $\tilde x = \thetatilde_t\mid \cF_{t-1}$, and $\cE=\cE_t$, we have 
\begin{equation*}
    \Pr\bs{O_t \mid  \cF_{t-1}, \widehat E_t} \geq 1/\br{4\sqrt{e\pi}}=:p.
\end{equation*}
As a result, we can upper bound the instantaneous regret as
\begin{align*}
    r_t^{\text{TS}}\ind(E_t \cap A_t) &\leq  \frac{ 2\beta_t(\delta')c(\delta'')\E\bs{\norm{\phi(\pi_t) - \phi(\pi'_t)}_{V_t^{-1}} \mid \cF_{t-1}}}{p}  + \varepsilon\\
    &\leq  \frac{ 2\beta_t(\delta')c(\delta'')\E\bs{\norm{\phi(\tau_t) - \phi(\tau'_t)}_{V_t^{-1}} \mid \cF_{t-1}}}{p} + \varepsilon,
\end{align*}
where the second inequality follows from the convexity of norms. Applying Lemma~\ref{lem:elliptical_lemma} 2), we have on the good event $G$ that
\begin{align}
    \sum_{t=1}^T r_t^{\text{TS}} &\leq \dfrac{ 2\beta_T(\delta')c(\delta'')}{p}\sum_{t=1}^T\E\bs{\norm{\phi(\tau_t) - \phi(\tau'_t)}_{V_t^{-1}} \mid \cF_{t-1}} + T\varepsilon  \label{eq:regret_bound_TS} \\
    &\leq \dfrac{ 2\beta_T(\delta')c(\delta'')}{p}C_T + T\varepsilon, \nonumber
\end{align}
for the constant
\begin{equation*}
    C_T=2 \sqrt{T \br{2 d \log\left(1 \!+\! \frac{4TL^2\effH^2}{d \lambda}\right) \!\!+\! \frac{12 d L^2\effH^2}{\log(2) \lambda} \log\left(1\!+\!\frac{4L^2\effH^2}{\log(2) \lambda}\right)}}
    + \frac{16L\effH}{\sqrt{\lambda}} \log\br{\frac{2}{\delta'}}.
\end{equation*}

\textit{Step 4 (bounding ${r_t^{\text{MLE}}}$)}: Conditioned on event $E_t$ we have
\begin{align*}
    r_t^{\text{MLE}} &= \ip{\tilde\theta_t - \theta^*}{\phi(\pi_t) - \phi(\pi'_t)} \nonumber\\
    &=\ip{\tilde\theta_t - \hat\theta_t}{\phi(\pi_t) - \phi(\pi'_t)} + \ip{\hat\theta_t-\theta^*}{\phi(\pi_t) - \phi(\pi'_t)} \nonumber\\
    &\leq \beta_t(\delta')(1+c(\delta''))\norm{\phi(\pi_t) - \phi(\pi'_t)}_{V_t^{-1}} \\
    &\le \beta_t(\delta')(1+c(\delta'')) \E\bs{\norm{\phi(\tau_t) - \phi(\tau'_t)}_{V_t^{-1}} | \cF_{t-1}^{\theta}},
\end{align*}
where the last inequality follows from the convexity of norms. From here, we can again apply the second result of Lemma \ref{lem:elliptical_lemma}. We deduce that on the good event $G$, we have
\begin{align*}
    \sum_{t=1}^T r_t^{\text{MLE}} &\leq \beta_T(\delta')(1+c(\delta''))\sum_{t=1}^T\E\bs{\norm{\phi(\tau_t) - \phi(\tau'_t)}_{V_t^{-1}} | \cF_{t-1}^{\theta}}\\
    &\leq\beta_T(\delta')(1+c(\delta'')) C_T.
\end{align*}

In summary, we can conclude that with probability at least $1-\delta$, the regret can be bounded as follows
\begin{align*}
    R(T) &\leq \sum_{t=1}^T \br{r_t^{\text{TS}} + r_t^{\text{MLE}}} + T\varepsilon + BL\effH\\
    &\leq \br{\dfrac{ 2\beta_T(\delta')c(\delta'')}{p} +\beta_T(\delta')(1+c(\delta'')) }C_T + 2T\varepsilon + BL\effH\\
    &\leq \dfrac{ 3\beta_T(\delta')c(\delta'')}{p}C_T + 2T\varepsilon + BL\effH\\
    &= \dfrac{ 3\bs{\sqrt{\kappa\bs{\log\br{\frac{5}{\delta}} + d \log \br{\max\bc{e, \frac{4eBL\effH (T-1)}{d}}}}} + 2\sqrt{\lambda}B}\sqrt{2d\log(10dT/\delta)}}{p}\\
    &\cdot \bs{2 \sqrt{T \br{2 d \log\left(1 \!+\! \frac{4TL^2\effH^2}{d \lambda}\right) \!\!+\! \frac{12 d L^2\effH^2}{\log(2) \lambda} \log\left(1\!+\!\frac{4L^2\effH^2}{\log(2) \lambda}\right)}}
    + \frac{16L\effH}{\sqrt{\lambda}} \log\br{\frac{10}{\delta}}}\\
    &+ 2\sqrt{T}+ BL\effH\\
    &= \cO\br{ \sqrt{ \kappa d^3 T\log\br{\dfrac{dT}{\delta}}^3 }}.
\end{align*}


\end{proof}

\section{Proof of Theorem \ref{thm:last_iterate}}
\label{app:proof_last_iterate}
\lastiterate*
\begin{proof}
    Similar to the proof of Theorem~\ref{thm:regret_lazy_thompson_sampling}, we start by defining a set of good events.

    \textit{Step 1 (good events)}: Recall that in Algorithm~\ref{alg:explore} we set $\varepsilon=1/\sqrt{T}$ and $\delta'=\delta/5$; at round $t$, conditional on $\cF_{t-1}=\sigma(x_1,\hdots, x_{t-1})$, we sample $\thetatilde_t\sim\cN(0, V_t^{-1})$ with $V_t = \lambda I + \sum_{k=1}^{t-1} x_k x_k^\top$ (where $x_k = \phi(\tau_k) - \phi(\tau_k')$), and $\thetahat$ is the MLE estimate at time $T+1$ using all the data from round $1$ to $T$. We now define the following high probability events. 
\begin{enumerate}
    \item Let $\delta'' = \delta'/T$ and $c(\delta''):= \sqrt{2d\log(2d/\delta'')} = \sqrt{2d\log(10dT/\delta)}$. Consider the centered ellipsoid
        \begin{equation*}
            \bar{\cE}_t^{\text{TS}} := \bc{\theta\in\R^d: \norm{\theta}_{V_t}\leq c(\delta'')}.
        \end{equation*}
        We define the events
        \begin{equation*}
            \widehat E:=\bc{\|\theta^* - \thetahat\|_{V_{T+1}} \leq \beta_{T+1}(\delta')},\quad \widetilde E_t:=
                \bc{\tilde \theta_t\in \bar{\cE}_t^{\text{TS}}},
        \end{equation*}
        and let $G_1:= \bigcap_{t=1}^T \widetilde E_t \cap \widehat E$. This event implies that $\theta^*$ lies in the confidence ellipsoid $\cE_{T+1}(\delta')$, and the sampled parameter lies in $ \bar{\cE}_t^{\text{TS}}$ for all $t\in[1,T]$.
    \item Let $G_2$ denote the event that for
    {\small
    \begin{equation*}
        C_T:=2 \sqrt{T \br{2 d \log\left(1 \!+\! \frac{4TL^2\effH^2}{d \lambda}\right) \!\!+\! \frac{12 d L^2\effH^2}{\log(2) \lambda} \log\left(1\!+\!\frac{4L^2\effH^2}{\log(2) \lambda}\right)}}
        + \frac{16L\effH}{\sqrt{\lambda}} \log\br{\frac{2}{\delta'}},
    \end{equation*}
    }
    it holds that
    \begin{equation*}
        \sum_{t=1}^T\E\bs{\norm{\phi(\tau_t) - \phi(\tau_t')}_{V_t^{-1}} \mid \cF_{t-1}} \leq C_T.
    \end{equation*}
    \item Let $A_t = \bc{V_{\thetatilde_t}^* - V_{\thetatilde_t}^{\pi_t}\leq \varepsilon}$, where $\pi_t= \oracle\br{\thetatilde_t, \varepsilon, \delta'/T}$, and let $G_3 = \bigcap_{t=1}^T A_t$.
    \item Let $G_4$ denote the event that $V_{\thetahat}^* - V_{\thetahat}^{\pihat}\leq \varepsilon$ where $\pihat = \oracle\br{\thetahat, \varepsilon, \delta'}$.
\end{enumerate}
Analogously to Proposition~\ref{app:prop:good_event}, it follows that $\Pr[G_1^\complement]\leq 2\delta'$ and $\Pr[G_k^\complement]\leq \delta'$ for $k=2,3,4$. Thus, by a union bound, the good event $G:= \bigcap_{i=1}^4 G_i$ happens with probability at least $1-\delta$.

\textit{Step 2 (suboptimality bound)}: 
    In the following, we will denote $\pi^*\in\argmax_\pi \ip{\theta^*}{\phi(\pi)}$ for an arbitrary optimal policy corresponding to the ground truth reward parameter $\theta^*$. Under the above good event $G$, we can bound the suboptimality of $\pihat$ as follows
   \begin{align*}
        \subopt(\pihat) &= \ip{\theta^*}{\phi(\pi^*) - \phi(\pihat)} \leq \ip{\theta^*-\thetahat}{\phi(\pi^*) - \phi(\pihat)} + \varepsilon \\
        &\leq \max_{\pi, \pi'} \ip{\theta^*-\thetahat}{\phi(\pi) - \phi(\pi')} + \varepsilon \leq \beta_{T+1}(\delta') \max_{\pi, \pi'} \norm{\phi(\pi) - \phi(\pi')}_{V_{T+1}^{-1}} + \varepsilon. \numberthis \label{eq:app:g-optimal}
    \end{align*}
    Next, we continue with bounding the term $\max_{\pi, \pi'} \norm{\phi(\pi) - \phi(\pi')}_{V_{T+1}^{-1}}$. 

    \textit{Step 3 (approximate G-optimal design)}:
    Due to the matrix inequalities $V_{T+1} \succeq V_{T} \succeq \hdots \succeq V_{1}$, we have
    \begin{equation*}
        \norm{\cdot}_{V_{T+1}^{-1}}\leq \norm{\cdot}_{V_{T}^{-1}}\leq \hdots \leq \norm{\cdot}_{V_{1}^{-1}}.
    \end{equation*}
    Therefore, we can bound the right-hand-side of \eqref{eq:app:g-optimal} as follows
    \begin{align*}
         \max_{\pi, \pi'} \norm{\phi(\pi) - \phi(\pi')}_{V_{T+1}^{-1}} &\leq \frac{1}{T} \sum_{t=1}^T\max_{\pi, \pi'}\norm{\phi(\pi) - \phi(\pi')}_{V_{T+1}^{-1}}\\
         &\leq \frac{1}{T} \sum_{t=1}^T \max_{\pi, \pi'}\norm{\phi(\pi) - \phi(\pi')}_{V_{t}^{-1}}\\
         &\leq \frac{2}{T} \sum_{t=1}^T \max_{\pi}\norm{\phi(\pi) - \phi(\pi_t')}_{V_{t}^{-1}},
    \end{align*}
    where we used the triangle inequality in the last step. If it were tractable to compute $\pi_t = \argmax_{\pi} \norm{\phi(\pi)- \phi(\pi_t')}_{V_t^{-1}}$, we could invoke Lemma~\ref{lem:elliptical_lemma} directly from here. Instead, we leverage randomized exploration to keep our algorithm tractable. In particular, we consider the function $f_t(\theta) = \max_{\pi} \ip{\theta}{\phi(\pi) - \phi(\pi_t')}$, the ellipsoid $\bar{\cE}_t=\bc{\theta: \norm{\theta}_{V_t}\leq 1}$, and the event $O_t = \{f_t(\thetatilde_t)\geq \max_{\theta\in\bar{\cE}_t} f_t(\theta)\}$. By Proposition~\ref{prop:convex} $f_t$ is convex and continuous. Applying Lemma~\ref{lem:anticonc} with $f=f_t$, $\thetatilde = \thetatilde_t \mid \cF_{t-1}$, and $\cE = \bar{\cE}_t$, it holds that $\Pr\bs{O_t \mid \cF_{t-1}}\geq p:=1/(4\sqrt{e\pi})$. Hence, we can proceed analogously to the proof of Theorem~\ref{thm:regret_thompson_sampling}:
    \begin{align*}
        \max_{\pi} \norm{\phi(\pi)-\phi(\pi_t')}_{V_t^{-1}}\ind\br{\widetilde E_t  \cap A_t} &\stackrel{(i)}{=} \max_{\theta\in\bar{\cE}_t} f_t(\theta)\ind\br{\widetilde E_t \cap A_t}\\
        &\leq  \E\bs{f_t(\thetatilde_t)\ind\br{\widetilde E_t  \cap A_t} \mid O_t, \cF_{t-1}}\\
        &\leq \E\bs{\ip{\thetatilde_t}{\phi(\pi_t) - \phi(\pi_t')}\ind\br{\widetilde E_t  \cap A_t} \mid O_t, \cF_{t-1}} + \varepsilon\\
        &\leq c(\delta'')\E\bs{\norm{\phi(\pi_t) - \phi(\pi_t')}_{V_t^{-1}} \mid O_t, \cF_{t-1}} + \varepsilon\\
        &\stackrel{(ii)}{\leq} \frac{c(\delta'')}{p} \E\bs{\norm{\phi(\pi_t) - \phi(\pi_t')}_{V_t^{-1}} \mid \cF_{t-1}} + \varepsilon\\
        &\leq \frac{c(\delta'')}{p} \E\bs{\norm{\phi(\tau_t) - \phi(\tau_t')}_{V_t^{-1}} \mid \cF_{t-1}} + \varepsilon.
    \end{align*}
    Here, $(i)$ follows because $\bar{\cE}_t$ is a centered ellipsoid, and $(ii)$ from the total law of probability. 

    Using Lemma~\ref{lem:elliptical_lemma}, we conclude that with probability $1-\delta$, we have that
    \begin{align*}
        &\subopt(\pihat) \leq \varepsilon(1+ 2\beta_{T+1}(\delta')) +  \frac{2\beta_{T+1}(\delta')c(\delta'')}{pT}\underbrace{\sum_{t=1}^T\E\bs{\norm{\phi(\tau_t) - \phi(\tau_t')}_{V_t^{-1}} \mid \cF_{t-1}}}_{\leq C_T}\\
        \leq& \bs{\sqrt{\kappa\bs{\log\br{\frac{5}{\delta}} + d \log \br{\max\bc{e, \frac{4eBL\effH T}{d}}}}} + 2\sqrt{\lambda}B} \Bigg\{ \dfrac{3}{\sqrt{T}} + \dfrac{ 2\sqrt{2d\log(10dT/\delta)}}{pT} \\
     &\cdot \bs{2 \sqrt{T \br{2 d \log\left(1 \!+\! \frac{4TL^2\effH^2}{d \lambda}\right) \!\!+\! \frac{12 d L^2\effH^2}{\log(2) \lambda} \log\left(1\!+\!\frac{4L^2\effH^2}{\log(2) \lambda}\right)}}   
    + \frac{16L\effH}{\sqrt{\lambda}} \log\br{\frac{10}{\delta}}} \Bigg\}\\
    =&\; \cO\br{ \sqrt{ \dfrac{\kappa d^3}{T}\log\br{\dfrac{dT}{\delta}}^3 }}.
    \end{align*}
\end{proof}

\begin{remark}[Connection to G-optimal design]
Given a subset $\cX\subset\R^d$, a G-optimal design selects $x_1,\hdots, x_T\in \cX$ to minimize $\max_{x\in\mathcal X} \norm{x}_{V^{-1}}$ with $V=\sum_{t=1}^T x_t x_t^\top$. Under compactness and if $\operatorname{span}(\cX) = \R^d$, the Kiefer-Wolfowitz theorem \citep{kiefer1960equivalence} yields the lower bound $\max_{x\in\mathcal X} \norm{x}_{V^{-1}} \geq \sqrt{d/T}$, which is tight up to rounding. In the proof above, we show that Algorithm~\ref{alg:explore} guarantees $\max_{\pi, \pi'} \norm{\phi(\pi) - \phi(\pi')}_{V_{T+1}^{-1}} = \cO(d \log(dT/\delta)/\sqrt{T})$. Thus, ignoring the noise in $\phi(\tau)$ (handled via Freedman) and, for simplicity, the regularization $\lambda$, Algorithm~\ref{alg:explore} achieves a $\cO(\sqrt{d} \log(dT/\delta))$-approximate G-optimal design with probability at least $1-\delta$. 
\end{remark}

\section{Proof of Theorem \ref{thm:regret_lazy_thompson_sampling}}
\label{app:proof_regret_lazy_thompson_sampling}
\lazy*

We start the proof by showing that the number of rounds where the design matrix is updated is small.
\begin{lemma}[Number of design matrix updates]
    Using Algorithm \ref{alg:lazy_thompson_sampling} with a parameter $C>0$, it holds that:
    \begin{equation*}
        \sum_{t=1}^T \ind\bs{V_{t+1} \neq V_t} \le \frac{d}{\log\br{1+C}}\log\br{1 + \frac{T(2L\effH)^2}{d\lambda}}.
    \end{equation*}
    That is, the number of updates of the matrix $V_t$ is at most logarithmic in the number of interactions $T$.
\end{lemma}

\begin{proof}
    Denote $N_T = \sum_{t=1}^T \ind\bs{V_{t+1} \neq V_t}$, and let $\ts \leq T$ be the last time the matrix $V_t$ was updated. Then,
    \begin{align*}
        \det(V_T) &= \det(V_{\ts}) \ge (1+C)^{N_T} \det(\lambda I).
    \end{align*}
    Here, we used that for two consecutive update rounds $\ts \geq \ts'$ we have that $\det V_{\ts} \ge (1+C) \det V_{\ts'}$. Then, using the trace-determinant inequality $\det A \leq \br{\tr(A)/d}^d$, we have
    \begin{equation*}
        \br{1+C}^{N_T} \lambda^d \le \br{\frac{d\lambda + T (2L\effH)^2}{d}}^d,
    \end{equation*}
    and
    \begin{align*}
        N_T \le \frac{d}{\log\br{1+C}} \log\br{1+\frac{T (2L\effH)^2}{d\lambda}}.
    \end{align*}
\end{proof}

We now present the proof for the regret bound, which proceeds similarly to that of Theorem \ref{thm:regret_thompson_sampling} up to some modifications. The first change is in the regret decomposition, which needs to be adapted because we no longer compare to the past policy but rather to $\pi_t' = \oracle(\thetahat_{\ts(t)}, \varepsilon, \delta)$. The second change is using Lemma \ref{lem:lazy_elliptical_lemma} instead of Lemma \ref{lem:elliptical_lemma} to bound the sum of norms of trajectory features. Finally, the good events defined in the proof of Theorem \ref{thm:regret_thompson_sampling} are slightly modified to account for the lazy design matrix and the new choice of comparator policy $\pi_t'$.

\begin{proof}
    Let us first define the function $\ts: \bN \to \bN$, which to a time $t$, assigns the last time $\ts(t) \le t$ that the update condition (line 4 in Algorithm \ref{alg:lazy_thompson_sampling}) was met.

    We work with the same two filtrations as before $\cF_{t-1}=\sigma(x_1,y_1,\hdots, x_{t-1}, y_{t-1})$ and $\cF_{t-1}^\theta=\sigma(\cF_{t-1}, \thetatilde_t)$. And we recall that, given $\cF_{t-1}$, $\thetatilde_t$ in Algorithm \ref{alg:lazy_thompson_sampling} is sampled from $\cN(\thetahat_{\ts(t)}, \beta_{\ts(t)}^2 V_{\ts(t)}^{-1})$.

    \textit{Step 1 (good events)}: Consider $\delta'=\delta/4$, we redefine the high-probability events:
    \begin{enumerate}
        \item Let $\delta'' = \delta'/T$ and $c(\delta''):= \sqrt{2d\log(2d/\delta'')}$. Consider the inflated ellipsoid
            \begin{equation*}
                \cE_{\ts(t)}^{\text{TS}} := \bc{\theta\in\R^d: \norm{\theta - \hat\theta_{\ts(t)}}_{V_{\ts(t)}}\leq \beta_{\ts(t)}(\delta')c(\delta'')}.
            \end{equation*}
            We define the events $E_t:= \widehat E_t \cap \widetilde E_t, \quad \widehat E_t:=\bc{\theta^* \in \cE_{\ts(t)}},\quad \widetilde E_t:= \bc{\tilde \theta_t\in \cE_{\ts(t)}^{\text{TS}}},$ and let $G_1 := \bigcap_{t=1}^T E_t$.
        \item Let {\small $C_T:=2
\sqrt{\dfrac{8(1+C)\,d T\,L^2\effH^2}{\lambda}\log\!\Bigl(1 + \frac{4T L^2\effH^2}{\lambda d}\Bigr)}+ \frac{16L\effH}{\sqrt{\lambda}} \log(2/\delta)$}, and define $G_2$ as the event under which it holds that
        \begin{equation*}
            \sum_{t=1}^T\E\bs{\norm{\phi(\tau_t) - \phi(\tau_t')}_{V_{\ts(t)}^{-1}} \mid \cF_{t-1}} \leq C_T,
        \end{equation*}
        and 
        \begin{equation*}
            \sum_{t=1}^T\E\bs{\norm{\phi(\tau_t) - \phi(\tau_t')}_{V_{\ts(t)}^{-1}} \mid \cF_{t-1}^\theta} \leq C_T.
        \end{equation*}
        \item Let $A_t = \bc{V_{\thetatilde_t}^* - V_{\thetatilde_t}^{\pi_t}\leq \varepsilon}$, where $\pi_t= \oracle\br{\thetatilde_t, \varepsilon, \delta'/T}$, and let $G_3 = \bigcap_{t=1}^T A_t$.
    \end{enumerate}
    We also define the intersection, $G:=\bigcap_{i=1}^3 G_i$, of all good events. 
    
    We now compare each of these events to their counterparts in the proof of Theorem \ref{thm:regret_thompson_sampling}. The event $G_1$ is modified because we replace $V_t$ by $V_{\ts(t)}$ and $\hat\theta_t$ by $\hat\theta_{\ts(t)}$. The event $G_1$ still holds with probability at least $1-2\delta'$ using the same concentration arguments as before. The event $G_2$ is also modified to account for the lazy design matrix, and it holds with probability $1-\delta'$ thanks to Lemma \ref{lem:lazy_elliptical_lemma}. Finally, the event $G_3$ remains unchanged.
    
    We conclude that $G$ happens with probability at least $1-\delta$.

    \textit{Step 2 (regret decomposition):} Since Algorithm \ref{alg:lazy_thompson_sampling} uses a different comparator policy $\pi_t'$ than Algorithm~\ref{alg:regret}, we derive a new regret decomposition. On the good event $G$, we have
    \begin{align*}
        R(T) &= \frac{1}{2}\sum_{t=1}^T \br{2V_{\theta^*}^*-V_{\theta^*}^{\pi_t}-V_{\theta^*}^{\pi'_t}}\\
        &= \frac{1}{2}\sum_{t=1}^T \br{(V_{\theta^*}^*-V_{\theta^*}^{\pi_t}) + \ip{\theta^*}{\phi(\pi^*)-\phi(\pi_t')} }\\
        &= \frac{1}{2}\sum_{t=1}^T \br{(V_{\theta^*}^*-V_{\theta^*}^{\pi_t}) + \ip{\theta^*}{\phi(\pi^*)-\phi(\pi_t)} + \ip{\theta^*}{\phi(\pi_t)-\phi(\pi_t')} }\\ 
        &= \frac{1}{2}\sum_{t=1}^T \br{2 (V_{\theta^*}^*-V_{\theta^*}^{\pi_t}) + \ip{\theta^*-\thetahat_{\ts(t)}}{\phi(\pi_t)-\phi(\pi_t')} + \ip{\thetahat_{\ts(t)}}{\phi(\pi_t)-\phi(\pi_t')} }\\
        &\le \frac{1}{2}\sum_{t=1}^T \br{2 \underbrace{(V_{\theta^*}^*-V_{\theta^*}^{\pi_t})}_{r_t} + \norm{\theta^*-\thetahat_{\ts(t)}}_{V_{\ts(t)}} \norm{\phi(\pi_t)-\phi(\pi_t')}_{V_{\ts(t)}^{-1}} + \varepsilon },
    \end{align*}
    where the last line follows from the Cauchy-Schwarz inequality and because $\pi_t' = \oracle(\thetahat_{\ts(t)}, \varepsilon, \delta)$.

    The second term in the decomposition above can be bounded on the good event $G$ as:
    \begin{align*}
        \sum_{t=1}^T \norm{\theta^*-\thetahat_{\ts(t)}}_{V_{\ts(t)}} \norm{\phi(\pi_t)-\phi(\pi_t')}_{V_{\ts(t)}^{-1}} &\le \beta_T \sum_{t=1}^T \norm{\phi(\pi_t)-\phi(\pi_t')}_{V_{\ts(t)}^{-1}} \\
        &\le \beta_T C_T
    \end{align*}
    where $C_T = 2
\sqrt{\dfrac{8(1+C)\,d T\,L^2\effH^2}{\lambda}\log\!\Bigl(1 + \frac{4T L^2\effH^2}{\lambda d}\Bigr)}+ \frac{16L\effH}{\sqrt{\lambda}} \log(2/\delta)$.

    For the first term in the regret decomposition, similarly to Appendix \ref{app:proof_regret_thompson_sampling}, we have that:
    \begin{align*}
        r_t = V_{\theta^*}^*-V_{\theta^*}^{\pi_t} = \underbrace{\Delta_t(\theta^*) - \Delta_t(\tilde\theta_t)}_{r_t^{\text{TS}}} + \underbrace{\ip{\tilde\theta_t - \theta^*}{\phi(\pi_t) - \phi(\pi'_t)}}_{r_t^{\text{MLE}}}
    \end{align*}
    where we recall the gap function $\Delta_t(\theta):= \max_{\pi}\ip{\theta}{\phi(\pi)-\phi(\pi'_t)}$.

    \textit{Step 3 (bounding ${r_t^{\text{TS}}}$)}: The proof for $\sum_t r_t^{\text{TS}}$ proceeds exactly like Appendix \ref{app:proof_regret_thompson_sampling} up to Equation \eqref{eq:regret_bound_TS}, this is because the probability of the optimism event $O_t$ and the events $E_t$ and $\hat{E}_t$ is unaffected by the change to the algorithm. Then, we have that:
    \begin{equation*}
        R^{\text{TS}}(T) = \sum_{t=1}^T r_t^{\text{TS}} \leq \dfrac{ 2\beta_T(\delta')c(\delta'')}{p}\sum_{t=1}^T\E\bs{\norm{\phi(\tau_t) - \phi(\tau'_t)}_{V_t^{-1}} \mid \cF_{t-1}} + T\varepsilon.
    \end{equation*}
    We can then conclude, on the good event $G$, that
    \begin{equation*}
        R^{\text{TS}}(T) \le \dfrac{2\beta_T(\delta')c(\delta'')}{p}C_T + T\varepsilon.
    \end{equation*}
    
    \textit{Step 4 (bounding ${r_t^{\text{MLE}}}$)}: This step is analogous to Appendix~\ref{app:proof_regret_thompson_sampling}. We have:
    \begin{align*}
        \sum_{t=1}^T r_t^{\text{MLE}} &\le \beta_T(\delta')(1+c(\delta'')) \sum_{t=1}^T \norm{\phi(\pi_t) - \phi(\pi'_t)}_{V_t^{-1}} \\
        &\le \beta_T(\delta')(1+c(\delta'')) \sum_{t=1}^T \bE[\norm{\phi(\tau_t) - \phi(\tau'_t)}_{V_t^{-1}} | \cF_{t-1}^\theta],\\
        &\le \beta_T(\delta')(1+c(\delta'')) C_T
    \end{align*}
    where the second inequality follows from the convexity of the norm.

    In summary, we conclude that with probability at least $1-\delta$, the regret can be bounded as follows:
    \begin{align*}
        R(T) &\le \frac{1}{2}\sum_{t=1}^T \br{2 \underbrace{(V_{\theta^*}^*-V_{\theta^*}^{\pi_t})}_{ = r_t^{\text{TS}} + r_t^{\text{MLE}}} + \underbrace{\norm{\theta^*-\thetahat_{\ts(t)}}_{V_{\ts(t)}}}_{\le \beta_T} \norm{\phi(\pi_t)-\phi(\pi_t')}_{V_{\ts(t)}^{-1}} + \varepsilon } \\
        &\le \dfrac{2\beta_T(\delta')c(\delta'')}{p}C_T + T\varepsilon + \beta_T(\delta')(1+c(\delta'')) C_T + \frac{\beta_T}{2} C_T + T \varepsilon.
    \end{align*}

\end{proof}

\section{Technical results}
\label{app:sec:technical_results}

This section presents the technical results necessary for our theorems' proofs.

\subsection{Confidence sequence}
The first result is a confidence sequence for the maximum likelihood estimation. It is an elliptical relaxation of the likelihood ratio confidence sequence provided in Theorem 3.1 of \citet{lee2024unified}.

\confidence*
\begin{proof}
    By a first-order Taylor approximation with integral remainder, we have
    \begin{equation*}
        \cL_{\cD_t}(\theta^*) = \cL_{\cD_t}(\thetahat_t) + \ip{\nabla \cL_{\cD_t}(\thetahat_t)}{\theta^* - \thetahat_t} + (\theta^* - \thetahat_t)^\top G_t(\thetahat_t, \theta^*)(\theta^* - \thetahat_t),
    \end{equation*}
    where $\cL_{\cD_t}$ was defined in Equation \eqref{eq:mle} and
    \begin{align*}
        G_t(\thetahat_t, \theta^*) &= \int_{0}^1 (1-\tau) \br{\sum_{k = 1}^{t-1} \dot{\sigma}(\ip{\thetahat_t + \tau(\theta^* - \thetahat_t)}{x_k}) x_k x_k^\top} \diff \tau\\
        &=  \sum_{k = 1}^{t-1} \bs{\int_{0}^1 (1-\tau) \dot{\sigma}(\ip{\thetahat_t + \tau(\theta^* - \thetahat_t)}{x_k}) \diff \tau}  x_k x_k^\top\\
        &\succeq \kappa^{-1} \sum_{k=1}^{t-1}x_k x_k^\top. \numberthis\label{eq:kappa_loewner}
    \end{align*}
    Rearranging terms gives
    \begin{align*}
        \cL_{\cD_t}(\theta^*) - \cL_{\cD_t}(\thetahat_t) &\stackrel{(i)}{\geq} (\theta^* - \thetahat_t)^\top G_t(\thetahat_t, \theta^*)(\theta^* - \thetahat_t) \\
        &\stackrel{(ii)}{\geq} (\theta^* - \thetahat_t)^\top \br{\kappa^{-1} \sum_{k=1}^{t-1}x_k x_k^\top}(\theta^* - \thetahat_t) \\
        &= \kappa^{-1}\norm{\theta^* - \thetahat_t}_{V_t}^2 - \kappa^{-1}\lambda \norm{\theta^* - \thetahat_t}^2,
    \end{align*}
    where $(i)$ follows from the first order optimality condition for $\thetahat_t$ and $(ii)$ from the lower bound in $\eqref{eq:kappa_loewner}$. Rearranging again and applying Theorem 3.1 by \citet{lee2024unified} for likelihood ratio confidence sequences with the Lipschitz constant of $\cL_{\cD_t}$ equal to $L_t = 2L\effH (t-1)$, we get 
    \begin{equation*}
        \norm{\theta^* - \thetahat_t}_{V_t}^2 \leq \kappa\bs{\log\br{\frac{1}{\delta}} + d \log \br{\max\bc{e, \frac{4eBL\effH (t-1)}{d}}}} + 4\lambda B^2.
    \end{equation*}
\end{proof}

\subsection{Optimism with constant probability}

We first recall the following standard concentration and anti-concentration property of the Gaussian distribution. 

\begin{lemma}[Appendix A of \citep{pmlr-v54-abeille17a}]\label{lem:concentration_anticoncentration}
    Let $z\sim\cN(0, I)$ be a $d$-dimensional Gaussian random vector. Then, we have:
    \begin{enumerate}
        \item Anti-concentration: For any $u\in\cB^d(1)$, we have $\Pr\bs{\ip{u}{z} \geq 1} \geq \frac{1}{4\sqrt{e\pi}}$.
        \item Concentration: $\Pr\bs{\norm{z} \leq \sqrt{2d\log(2d/\delta)}} \geq 1-\delta$.
    \end{enumerate}
\end{lemma}

The anti-concentration property yields the following key result, which is required to prove a constant probability of optimism and subsequently control pessimism terms in the regret. While it has been proven by \citet{pmlr-v54-abeille17a} in their linear Thompson sampling analysis, we provide a concise proof based on convex analysis for completeness.

\begin{lemma}\label{lem:anticonc}
Let $f:\R^d\to\R$ be a continuous and convex function, and consider the ellipsoid $\cE:=\bc{\theta\in\R^d: \norm{\theta - \theta_0}_V \leq b}$ for a positive definite matrix $V$ and $b>0$. If $\thetatilde \sim \cN(\theta_0, b^2 V^{-1})$, then $\Pr\bs{f(\thetatilde)\geq \max_{\theta\in\cE}f(\theta)} \geq 1/\br{4\sqrt{e\pi}}$.
\end{lemma}
\begin{proof}
    Note that by definition of $\thetatilde$ we have $\thetatilde \stackrel{d}{=} \theta_0 + b V^{-1/2} \tilde z$, where $\tilde z\sim \cN(0, I)$. Hence, considering $g(z):= f(\theta_0 + b V^{-1/2} z)$, we have 
    \begin{equation*}
        p := \Pr\bs{f(\thetatilde)\geq \max_{\theta\in\cE}f(\theta)} = \Pr\bs{g(\tilde z)\geq \max_{z\in\cB^d(1)}g(z)},
    \end{equation*}
    where we used that $\theta_0 + b V^{-1/2} z \in \cE$ if and only if $z\in\cB^d(1)$.
    Since $g$ is a continuous convex function, we can choose $\bar z\in\argmax_{z\in\cB^d(1)} g(z)$ such that $\norm{\bar z}=1$. By optimality of $\bar z$ it holds that
\begin{equation}\label{eq:opt_condition}
    \partial g(\bar z) \subseteq N_{\cB^d(1)}(\bar z),
\end{equation}
where $N_{\cB^d(1)}(\bar z)= \bc{h\in\R^d: 0 \geq \ip{h}{z - \bar z}, \forall z\in\cB^d(1)} = \bc{h\in\R^d: h=\lambda \bar z, \lambda\geq 0}$ denotes the normal cone to $\cB^d(1)$ at $\bar z$. While this optimality condition for convex function maximization is somewhat standard (see \eg \citep[Theorem 32.4]{rockafellar1997convex}), we can verify directly that by optimality of $\bar z$ we have
\begin{align*}
    \partial g(\bar z) &:= \bc{h\in\R^d: g(z) - g(\bar z) \geq \ip{h}{z - \bar z}, \forall z}\\
    &\subseteq \bc{h\in\R^d: g(z) - g(\bar z) \geq \ip{h}{z - \bar z}, \forall z\in\cB^d(1)}\\
    &\subseteq \bc{h\in\R^d: 0 \geq \ip{h}{z - \bar z}, \forall z\in\cB^d(1)} =: N_{\cB^d(1)}(\bar z).
\end{align*}
Since $g$ is convex and finite on all of $\R^d$, we have $\partial g(\bar z)\neq \emptyset$. Therefore, the inclusion \eqref{eq:opt_condition} implies that
\begin{equation*}
    g(\tilde z) \geq g(\bar z) + \ip{\lambda \bar z}{\tilde z - \bar z} = g(\bar z) + \lambda(\ip{ \bar z}{\tilde z }-1), \quad \text{ for some } \lambda \geq 0.
\end{equation*}
Therefore, $\ip{\bar z}{\tilde z}\geq 1$ implies that $g(\tilde z) \geq g(\bar z)$, which yields the lower bound
\begin{equation*}
    p = \Pr\bs{g(\tilde z) \geq g(\bar z) } \geq \Pr\bs{\ip{\tilde z}{\bar z} \geq 1 }.
\end{equation*}
In light of Lemma~\ref{lem:concentration_anticoncentration}, this establishes $p \geq 1 / (4\sqrt{e\pi})$.
\end{proof}

The above Lemma is helpful in the following way: Let $\theta^*\in\cE$, $\thetatilde \sim \cN(\theta_0, b^2 V^{-1})$, and $x_{\thetatilde} = \argmax_{x\in\cX} \ip{\thetatilde}{x}$ for some bounded subset $\cX\subset \R^d$ (assuming the maximum exists). If $f$ is the support function of $\cX$, \ie $f(\theta) =  \max_{x\in\cX} \ip{\theta}{x}$, then we have with probability at least $1/\br{4\sqrt{e\pi}}$ that
\begin{equation*}
    \max_{x\in\cX} \ip{\theta^*}{x} = f(\theta^*) \leq \max_{\theta \in\cE} f(\theta) \leq  f(\thetatilde) = \ip{\thetatilde}{x_{\thetatilde}}.
\end{equation*}
Moreover, by part 2. of Lemma~\ref{lem:concentration_anticoncentration}, we still have that $\norm{\thetatilde-\theta_0}_V\leq \tilde\cO\br{\sqrt{d}b}$ with high probability. This idea will be key to the proofs of Theorems \ref{thm:regret_thompson_sampling}, \ref{thm:last_iterate}, and \ref{thm:regret_lazy_thompson_sampling}.

\subsection{Elliptical potential bounds}
\label{app:elliptical_potential}
Another key component of the convergence proofs are the following so-called elliptical potential lemmas that provide an upper bound on the sum of norms of sequentially observed vectors in the norm induced by their design matrix.
\begin{lemma}[Lemma 19.4 of \citet{lattimore2020bandit}]
\label{lem:standard_elliptical_lemma}
    Let $\br{x_t}_{t\ge 1} \subset \bR^d$ and for all $t\ge 1$, $\norm{x_t} \le L$, let $V_t = \lambda I + \sum_{k=1}^{t-1} x_k x_k^\top$ for some $\lambda>0$. Then, 
    \begin{align*}
        \sum_{t=1}^T \min\{1, \norm{ x_t }_{V_{t}^{-1}}^2 \} \le \: & 2 d \log\left(1+ \frac{T L^2}{d \lambda}\right).
    \end{align*}
\end{lemma}

In the analysis of our algorithms, a central challenge is to control the norms of policy feature differences $\phi(\pi_t) - \phi(\pi_t')$. However, the learner only observes the trajectory-level differences $x_t = \phi(\tau_t) - \phi(\tau_t')$, which are random variables with mean $\phi(\pi_t) - \phi(\pi_t')$. To overcome this, we build on Lemma \ref{lem:standard_elliptical_lemma} and introduce new tools to bound the sum of norms of policy feature differences. 

\begin{lemma}[Elliptical lemma]
\label{lem:elliptical_lemma}    
Let $\{x_t\}_{t\ge 1} \subset \R^d$ be a sequence of random vectors adapted to a filtration
$\{\cF_{t}\}_{t\ge 1}$, and let $\|x_t\|\le L$ almost surely for all $t\ge 1$. Let $V_t = \lambda I + \sum_{k=1}^{t-1} x_k x_k^\top$ for some $\lambda>0$. Then, the following holds:

\smallskip
\noindent
\textbf{1) Deterministic bound.} For all $T\in\bN$, almost surely,
\begin{align*}
    \sum_{t=1}^T \|x_t\|_{V_t^{-1}}^2
    \;\le\;
    2 d \log\!\Bigl(1+ \frac{T L^2}{d \lambda}\Bigr)
    \;+\;
    \frac{3 d L^2}{\lambda \log 2}\,
    \log\!\Bigl(1+\frac{L^2}{\lambda \log 2}\Bigr).
\end{align*}

\smallskip
\noindent
\textbf{2) High-probability bound.} For all $\delta\in (0,1)$, with probability at least $1-\delta$,
\begin{align*}
    \forall T\in\bN,\quad
    \sum_{t=1}^T \E\!\big[\|x_t\|_{V_t^{-1}} \,\big|\, \cF_{t-1}\big]
    &\le
    2\sqrt{T \br{2 d \log\left(1 \!+\! \frac{T L^2}{d \lambda}\right) \!\!+\! \frac{3 d L^2}{ \lambda \log 2} \log\left(1\!+\!\frac{L^2}{\lambda \log 2}\right)}} \\
    &\quad+ \frac{8 L}{\sqrt{\lambda}} \log(1/\delta).
\end{align*}
\end{lemma}

The first statement is a small improvement over Lemma \ref{lem:standard_elliptical_lemma} because it involves $\norm{ x_t }_{V_{t}^{-1}}^2$ instead of $\min\{1, \norm{ x_t }_{V_{t}^{-1}}^2\}$ and maintains a similar upper bound. In the second statement, $\{x_t\}_{t\ge1}$ represent trajectory features and their expected values are policy features. Hence, the second statement of the lemma above allows us to control the elliptical potentials of the policy features while only observing trajectory features.

\begin{proof}
    \textbf{First statement:} The proof of this result is based on the observation in \cite[Exercise 19.3]{lattimore2020bandit}. Namely, the number of times the term $\norm{x_t}_{{V}_{t}^{-1}}^2$ can be larger than one is at most $\frac{3 d}{\log(2)} \log(1+\frac{L^2}{\lambda \log(2)})$. 

    Let's define the rounds $\cT_T = \{t\le T, \norm{x_t}_{V_{t}^{-1}}^2 \le 1\}$, we have:
    \begin{align*}
        \sum_{t=1}^T \norm{x_t}_{V_{t}^{-1}}^2 &= \sum_{t\in \cT_T} \norm{x_t}_{V_{t}^{-1}}^2 + \sum_{t \notin \cT_T} \norm{x_t}_{V_{t}^{-1}}^2\\
        &\le \sum_{t\in \cT_T} \min\{1, \norm{x_t }_{V_{t}^{-1}}^2 \} + \frac{3 d L^2}{\lambda \log(2)} \log(1+\frac{L^2}{\lambda \log(2)}),
    \end{align*}
    where the first term of the last inequality follows by definition of $\cT_T$. The second term follows because the number of times $1\le t \le T$ not in $\cT_T$ is at most $\frac{3 d}{\log(2)} \log(1+\frac{L^2}{\lambda \log(2)})$ as previously discussed, and because $\norm{x_t}_{V_{t}^{-1}}^2 \le L^2 / \lambda$. Then, the proof is concluded by bounding the first sum on the right-hand side using Lemma \ref{lem:standard_elliptical_lemma}.

    \textbf{Second statement:} The proof proceeds by using Lemma \ref{lem:freedman}. We have for any $\delta\in (0,1)$ that with probability $1-\delta$:
    \begin{align*}
        \sum_{t=1}^T \bE[\norm{x_t }_{V_{t}^{-1}} | \cF_{t-1}] &\le 2 \sum_{t=1}^T \norm{x_t }_{V_{t}^{-1}} + \frac{8 L}{\sqrt{\lambda}} \log(1/\delta)\\
        &\le 2 \sqrt{T \sum_{t=1}^T \norm{x_t }_{V_{t}^{-1}}^2} + \frac{8 L}{\sqrt{\lambda}} \log(1/\delta)\\
        &\le 2 \sqrt{T \br{2 d \log\left(1 \!+\! \frac{T L^2}{d \lambda}\right) \!\!+\! \frac{3 d L^2}{\log(2) \lambda} \log\left(1\!+\!\frac{L^2}{\log(2) \lambda}\right)}} + \frac{8 L}{\sqrt{\lambda}} \log(1/\delta),
    \end{align*}
    where the first inequality uses Lemma \ref{lem:freedman}, the second uses the Cauchy-Schwarz inequality, and the last follows from the first result of the lemma.
\end{proof}

We now present a variant of the elliptical potential lemma above, adapted for the case where the design matrix is updated with lazy updates and optimal design; see Algorithm \ref{alg:lazy_thompson_sampling} for more details.

\begin{lemma}[Lazy elliptical lemma]
\label{lem:lazy_elliptical_lemma}
Let $(x_t)_{t\ge 1} \subset \R^d$ be a sequence of random vectors adapted to a filtration
$(\cF_{t})_{t\ge 1}$, with $\|x_t\|\le L$ almost surely.  
Fix $\lambda\geq L^2$ and $C>0$.

Define $(V_t,W_t)_{t\ge 1}$ as in Algorithm~\ref{alg:lazy_thompson_sampling}:
\begin{itemize}
    \item $V_1 = W_1 = \lambda I$, $\cD_1 = \emptyset$, $\ts=1$.
    \item For $t=1,\dots,T$:
    \begin{itemize}
        \item If $\det(W_t) > (1+C)\det(V_{\ts})$: 
        \begin{itemize}
            \item Let $V_t = V_{t-1} + \sum_{x\in \cD_{\text{opt}}} x x^\top$ be such that $\det V_t \geq \det W_t$ and $\abs{\cD_{\text{opt}}} \leq \abs{\cD}$.
            \item Set $\ts = t$, $\cD = \emptyset$
        \end{itemize}
        \item Observe $x_t$ and append $\cD = \cD \cup \bc{x_t}$.
        \item Update
        \[
            W_{t+1} =
            \begin{cases}
            V_t + x_t x_t^\top, & \text{if } t = \ts,\\
            W_t + x_t x_t^\top, & \text{otherwise,}
            \end{cases} \quad V_{t+1} = V_t.
        \]
    \end{itemize}
\end{itemize}

Then, the following holds:

\medskip\noindent
\textbf{1) Deterministic bound.}
For all $T\in\bN$,
\begin{align*}
\sum_{t=1}^T \|x_t\|_{V_t^{-1}}^2
&\le 
\dfrac{2(1+C)\,d\,L^2}{\lambda}\log\!\Bigl(1 + \frac{T L^2}{\lambda d}\Bigr).
\end{align*}

\medskip\noindent
\textbf{2) High-probability bound.}
For a fixed $T\in\bN$ and any $\delta\in (0,1)$, with probability at least $1-\delta$,
\begin{align*}
\sum_{t=1}^T \E\bigl[\|x_t\|_{V_t^{-1}} \mid \cF_{t-1}\bigr]
&\le 2
\sqrt{\dfrac{2(1+C)\,d T\,L^2}{\lambda}\log\!\Bigl(1 + \frac{T L^2}{\lambda d}\Bigr)}+ \frac{8L}{\sqrt{\lambda}} \log(1/\delta).
\end{align*}
\end{lemma}

\begin{proof}
    \textbf{First statement:} First note that $V_t \succeq \lambda I$ for all $t$, hence
$\|x_t\|_{V_t^{-1}}^2 \le L^2/\lambda$. For non-update steps, $t \neq \ts$, we have $V_t \preceq W_t$ and $\det(W_t) \le (1+C)\det(V_t)$. Hence, by Lemma~\ref{lem:inequality_of_matrix_norms}
\begin{equation}
\label{eq:design_smaller_than_naive}
    \forall x \in \bR^d, \quad \norm{x}_{V_{t}^{-1}}^2 \le \frac{\det(V_{t}^{-1})}{\det(W_{t}^{-1})}\norm{x}_{W_{t}^{-1}}^2  \leq (1+C) \norm{x}_{W_{t}^{-1}}^2.
\end{equation}
Using the inequality $\min\{1,z\}\le 2\log(1+z)$ for $z\ge0$, we get 
\begin{align*}
    \min\{1,\|x_t\|_{V_t^{-1}}^2\}
&\leq
(1+C)\,\min\{1,\|x_t\|_{W_t^{-1}}^2\} \leq 2(1+C) \log\br{1+\|x_t\|_{W_t^{-1}}^2}\\
&= 2(1+C) \log \frac{\det\br{W_{t+1}}}{\det\br{W_t}},
\end{align*}
where the last equality follows from $W_{t+1} = W_t + x_t x_t^\top$ and the matrix determinant lemma (see \citep[Lemma 11]{abbasi2011improved}). 

For update steps, $t=\ts$, we have by construction $\det(V_t) \geq \det(W_t)$ and $W_{t+1} = V_t + x_t x_t^\top$. Therefore, with the same reasoning as above,
\begin{equation*}
     \min\{1,\|x_t\|_{V_t^{-1}}^2\} \leq 2 \log\br{1+\|x_t\|_{V_t^{-1}}^2} = 2\log \frac{\det\br{W_{t+1}}}{\det\br{V_t}} \leq 2\log \frac{\det\br{W_{t+1}}}{\det\br{W_t}}.
\end{equation*}

By summing over $t$ and telescoping, we therefore get
\begin{equation*}
    \sum_{t=1}^T  \min\{1,\|x_t\|_{V_t^{-1}}^2\} \leq 2(1+C) \log \frac{\det\br{W_{T+1}}}{\det\br{W_1}} \leq 2(1+C) d\log\br{1+\dfrac{TL^2}{\lambda d}}.
\end{equation*}
Using that $\norm{x_t}_{V_t^{-1}}^2 \leq (L^2/ \lambda) \min\{1,\|x_t\|_{V_t^{-1}}^2\}$ concludes the proof of the first result. 

    \textbf{Second statement:} Analogously to Lemma~\ref{lem:elliptical_lemma}, the second statement follows from the first statement and Lemma~\ref{lem:freedman}.
\end{proof}
\begin{remark}
    Note that we have not used the trick from \cite[Exercise 19.3]{lattimore2020bandit} in the proof above, as it doesn't trivially extend to our lazy setting. 
\end{remark}

\subsection{Miscellaneous}

\begin{proposition}\label{prop:convex}
    The function $f(\theta) = \sup_{\pi \in \br{\Delta_{\cA}}^\cS} \: \ip{\theta}{\phi(\pi)}$ is convex and continuous over $\R^d$.
\end{proposition}
\begin{proof}
    The function $f$ is the support function of the set $\cZ = \bc{\phi(\pi): \pi:\cS\to\Delta_{\cA}}$. To prove the convexity, note that for any $\eta\in(0,1)$, we have
    \begin{equation*}
        f(\eta\theta + (1-\eta) \theta') \leq \eta\sup_{z\in\cZ} \ip{\theta}{z} + (1-\eta)\sup_{z\in\cZ}\ip{\theta'}{z}\leq \eta f(\theta) + (1-\eta)f(\theta').
    \end{equation*}
    Furthermore, any convex function is continuous over the relative interior of its effective domain $\dom f = \bc{x: f(x)< \infty}$ (see \eg \citep[Theorem 10.1]{rockafellar1997convex}). Since $|f(\theta)|\leq \norm{\theta}L\effH$, this implies that $f$ is continuous over $\R^d$.
\end{proof}

\begin{lemma}[Lemma 12 of \citep{abbasi2011improved}]
\label{lem:inequality_of_matrix_norms} 
    Let $A, B$, and $C$ be positive semi-definite matrices such that $A = B+C$. Then, we have that:
    \begin{equation*}
        \sup_{x \neq 0}  \frac{x^T A x}{x^T B x} \le \frac{\det(A)}{\det(B)}.
    \end{equation*}
\end{lemma}

This next lemma is one of the versions of Freedman's inequality \citep{freedman1975tail}.
\begin{lemma}[Lemma 2 of \citet{zhu2022efficient}]
\label{lem:freedman}
    Let $\left(z_t\right)_{t\leq T}$ be a real-valued sequence of random variables adapted to a filtration $\cF_{t}$. If $0\leq z_t \leq B$ almost surely, then with probability at least $1-\delta$,
    \begin{equation*}
        \sum_{t=1}^T z_t \leq \frac{3}{2} \sum_{t=1}^T \mathbb{E}\left[z_t \mid \cF_{t-1}\right]+4 B \log \left(2 \delta^{-1}\right)
    \end{equation*}
    and
    \begin{equation*}
        \sum_{t=1}^T \mathbb{E}\left[z_t  \mid \cF_{t-1}\right] \leq 2 \sum_{t=1}^T z_t+8 B \log \left(2 \delta^{-1}\right).
    \end{equation*}
\end{lemma}


\section{Discussion and background}

\subsection{Background on optimal experimental design}
\label{ap:background_optimal_design}

Given a possibly infinite set of features $\cX \subset \bR^d$, a D-optimal design is defined as a distribution $\pi$ such that:
\begin{equation*}
    \pi \in \argmax_{\pi \in \Delta_{\cX}} \log\det\br{\sum_{x\in\cX} \pi(x) x x^\top}.
\end{equation*}
The Kiefer-Wolfowitz theorem, see \citep{kiefer1960equivalence}, shows that a D-optimal design also ensures that $\max_{x\in\cX} \norm{x}_{\br{\sum_{x\in\cX} \pi(x) x x^\top}^{-1}}^2 = d$. A direct consequence is that $\cX$ is a subset of the ellipsoid $\{x \in \bR^d: \norm{x}_{(\sum_{x\in\cX} \pi(x) x x^\top)^{-1}}^2 \le d\}$. The Kiefer-Wolfowitz theorem can be interpreted by saying that the set $\{x \in \bR^d: \norm{x}_{(\sum_{x\in\cX} \pi(x) x x^\top)^{-1}}^2 \le d\}$ is the minimum volume ellipsoid containing $\cX$, see \cite[Theorem 21.1]{lattimore2020bandit}. 

In the case where we have a budget of $n$ samples for the design, we can define D-optimal design as the best allocation $\{n_x\}_{x\in\cX} \in \bN$ of the $n$ samples such that $\log\det\br{\sum_{x\in\cX} n_x x x^\top}$ is maximized and $\sum_{x\in\cX} n_x = n$. Finding a D-optimal design with a budget of $n$ is a challenging combinatorial optimization problem. However, there is a key property of the $\log\det$ function that enables an efficient approximation scheme. Namely, for a set $\cX \subset \bR^d$, the set function $S \subset \cX \to \log\det\br{\lambda I + \sum_{x \in S} xx^\top}$ is submodular if $\lambda$ is greater than one. Submodularity is a property describing decreasing additional benefit, known as diminishing returns. Fortunately, a submodular set function can be approximately maximized using a greedy algorithm, \citep{nemhauser1978analysis}. Therefore, using Algorithm \ref{alg:greedy_optimal_design}, we know that if $\ts$ and $\ts'$ are two consecutive update times, then the design matrix satisfies: $\log\det\br{V_{\ts'}} \ge (1-1/e) \max_{S \subset \cX} \log\det\br{V_{\ts}+ \sum_{x \in S} xx^\top}$.

\subsection{On the intractability of optimistic approaches}
\label{app:sec:opt_approaches}

\textbf{Optimism for regret minimization} Optimism in the face of uncertainty is a widely used principle for regret minimization in reinforcement learning. In the bandit setting, optimistic algorithms can be applied directly and yield minimax-optimal regret bounds \citep{auer2002using}. For regret minimization in RLHF, an optimistic algorithm would choose the policy $\pi_t$ to maximize the upper confidence bound on the reward difference relative to a comparator policy $\pi_t'$:
\begin{equation*}
    \pi_t = \argmax_\pi \max_{\theta\in\cE_t} V_{\theta}^\pi - V_{\theta}^{\pi_t'} = \argmax_{\pi} V_{\thetahat_t}^\pi + \beta_t \norm{\phi(\pi) - \phi(\pi_t')}_{V_t^{-1}},
\end{equation*}
where $\cE_t$ is a confidence set for $\theta^*$. This leads to the following bound on the instantaneous regret:
\begin{align*}
    r_t &= \br{V^*_{\theta^*} - V^{\pi_t'}_{\theta^*}} - \br{V^{\pi_t}_{\theta^*} - V^{\pi_t'}_{\theta^*}}\\
    &\leq \br{\max_\pi \max_{\theta\in\cE_t} V_{\theta}^\pi - V_{\theta}^{\pi_t'}} - \br{V^{\pi_t}_{\theta^*} - V^{\pi_t'}_{\theta^*}}\\
    &= \br{V_{\thetatilde_t}^{\pi_t} - V_{\thetatilde_t}^{\pi_t'}} - \br{V^{\pi_t}_{\theta^*} - V^{\pi_t'}_{\theta^*}}\\
    &\leq \norm{\thetatilde_t - \theta^*}_{V_t} \norm{\phi(\pi_t) - \phi(\pi_t')}_{V_t^{-1}} \leq 2\beta_t(\delta') \norm{\phi(\pi_t) - \phi(\pi_t')}_{V_t^{-1}},
\end{align*}
where $\thetatilde_t\in\cE_t$. The cumulative regret can then be bounded using standard elliptical potential arguments (see Lemma~\ref{lem:elliptical_lemma}).

\textbf{Optimism for preference-free exploration} As is evident from the proof of Theorem~\ref{thm:last_iterate}, algorithms that solely maximize an exploration bonus of the form
\begin{equation*}
    \norm{\phi(\pi) - \phi(\pi_t')}_{V_t^{-1}} \quad \text{or} \quad  \E_{\tau\sim\bP_{\pi}, \tau'\sim \bP_{\pi_t'}}\norm{\phi(\tau)-\phi(\tau')}^p_{V_t^{-1}}, p=1,2,
\end{equation*}
are also effective for preference-free exploration. 

\textbf{Challenge of optimizing the exploration bonus} The key problem of the above approaches is that for trajectory-level feedback -- preferences or reward -- they lead to optimization problems over policies that cannot be framed in terms of state-action rewards. \citet{efroni2021reinforcement} therefore conjecture that exactly solving these problems is intractable even in tabular settings. This stands in contrast to standard reinforcement learning with state–action feedback, where optimistic algorithms, and approaches based on optimal design~\citep{wagenmaker2022beyond, mutny2023active}, involve (or can be reduced to) maximization of bonuses of the form
\begin{equation*}
\E_{(s,a)\sim\mu_{\pi}} \norm{\phi(s,a)}^p_{V^{-1}}, p=1,2,
\end{equation*}
which are linear in the occupancy measure $\mu_{\pi}$ and can be maximized via dynamic programming.

These computational challenges associated with exploration under trajectory-level feedback motivate alternative approaches, such as randomized exploration, which can provide near-optimal theoretical guarantees while remaining computationally tractable.

\section{Experiments}\label{app:sec:experiments}
\subsection{Gridworld environment}
As illustrated in Figure~\ref{fig:gridworld_environment} below, the gridworld consists of 36 grid cells and the initial state lies in the center. The agent can choose the actions up, down, left, right, and will deterministically move in that direction or stay if it hits a boundary. The reward features are one-hot features for the six boundary states in blue, and the ground truth reward is $0.5$ for two of these states as indicated in Figure~\ref{fig:gridworld_environment}. Moreover, rewards are discounted with $\gamma = 0.9$. Our gridworld implementation builds on the code by \citet{schlaginhaufen2023identifiability}.

\begin{figure}[h]
    \centering
    \includegraphics[width=0.4\textwidth]{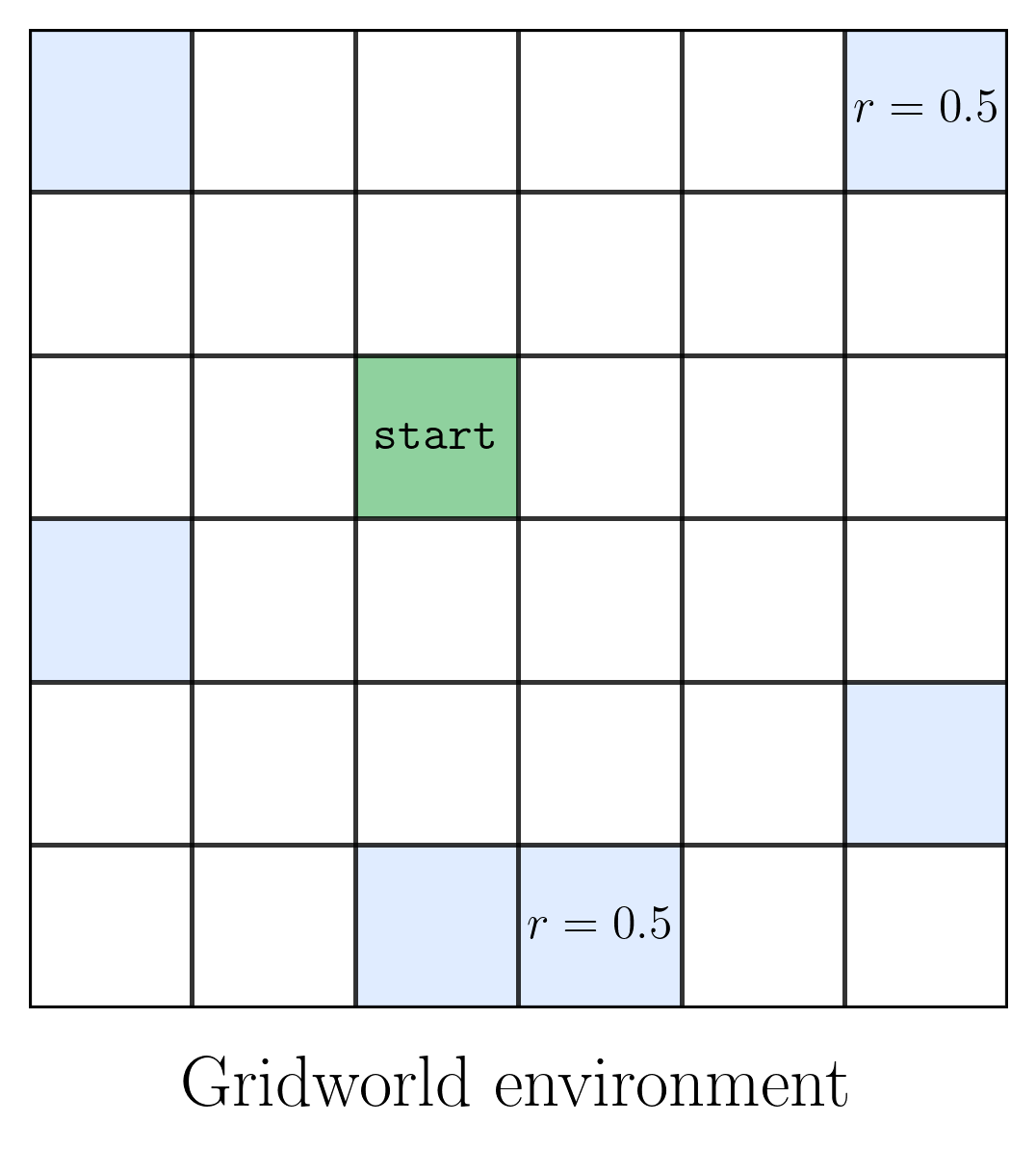}
  \caption{Illustration of the gridworld environment.}\label{fig:gridworld_environment}
\end{figure}

\subsection{Cartpole environment}
Here, we provide additional details for experiments on Isaac Lab's \texttt{Isaac-Cartpole-v0} environment, as well as experimental results for the pure exploration version of our algorithm.
\paragraph{Implementation details}
All experiments run on Isaac Lab’s unmodified \texttt{Isaac-Cartpole-v0} environment using the default PPO configuration. As reward features, we use the pre-defined reward terms. For cartpole these are: 1) Alive term: equal to 1 for all non-terminal states; 2) Termination term: equal to 1 for terminal states. A state is terminal if the cart goes out of bounds; 3) Goal tracking term: absolute value of pole angle measured from the upright position; 4) Cart velocity term: absolute value of cart velocity; 5) Joint velocity term: absolute value of pole angular velocity. However, our implementation supports any Isaac Lab manager‑based tasks.

We train over 30 RLHF iterations, using 30 steps of PPO at each iteration, and training is repeated for 20 independent seeds.
For the randomized exploration, we set $\beta_t = 0.001 + 0.1\max(1,\log t)$ and $\lambda=1$, and for lazy updates we set $C = 0.5$. At each RLHF iteration we compare 100 independently sampled trajectories. For the maximum likelihood estimation we perform 50 Adam steps (batch size $64$, $\ell_2$ penalty $\lambda = 10^{-1}$). 
Experiments were executed on a single machine equipped with an Intel i9-14900KS CPU and an NVIDIA RTX 4090 GPU; completing 30 RLHF iterations required approximately 2\,min\,50\,s.

\paragraph{Preference-free exploration}
Our results for the preference-free exploration algorithm \texttt{RPO-Explore} are shown in Figure~\ref{fig:experiments_pure_exploration} and \ref{fig:experiments_pure_exploration2} below. In Figure~\ref{fig:experiments_pure_exploration}, we see that all three versions of \texttt{RPO-Explore}~achieve performance competitive to RL with the ground truth reward\footnote{Note that the RL baseline makes $24\times 4096\times 50 = 4'915'200$ queries to the ground truth reward, whereas \texttt{RPO} uses at most 150 binary preference queries.}, but \texttt{RPO-OD-Explore}~needs the least preference queries. Moreover, in Figure~\ref{fig:experiments_pure_exploration2} we see that the performance during training is poor, \ie the regret is large, which is to be expected due to the pure exploration scheme.

\begin{figure}[h]
\centering
\begin{subfigure}[t]{0.48\textwidth}
    \centering
    \includegraphics[width=\textwidth]{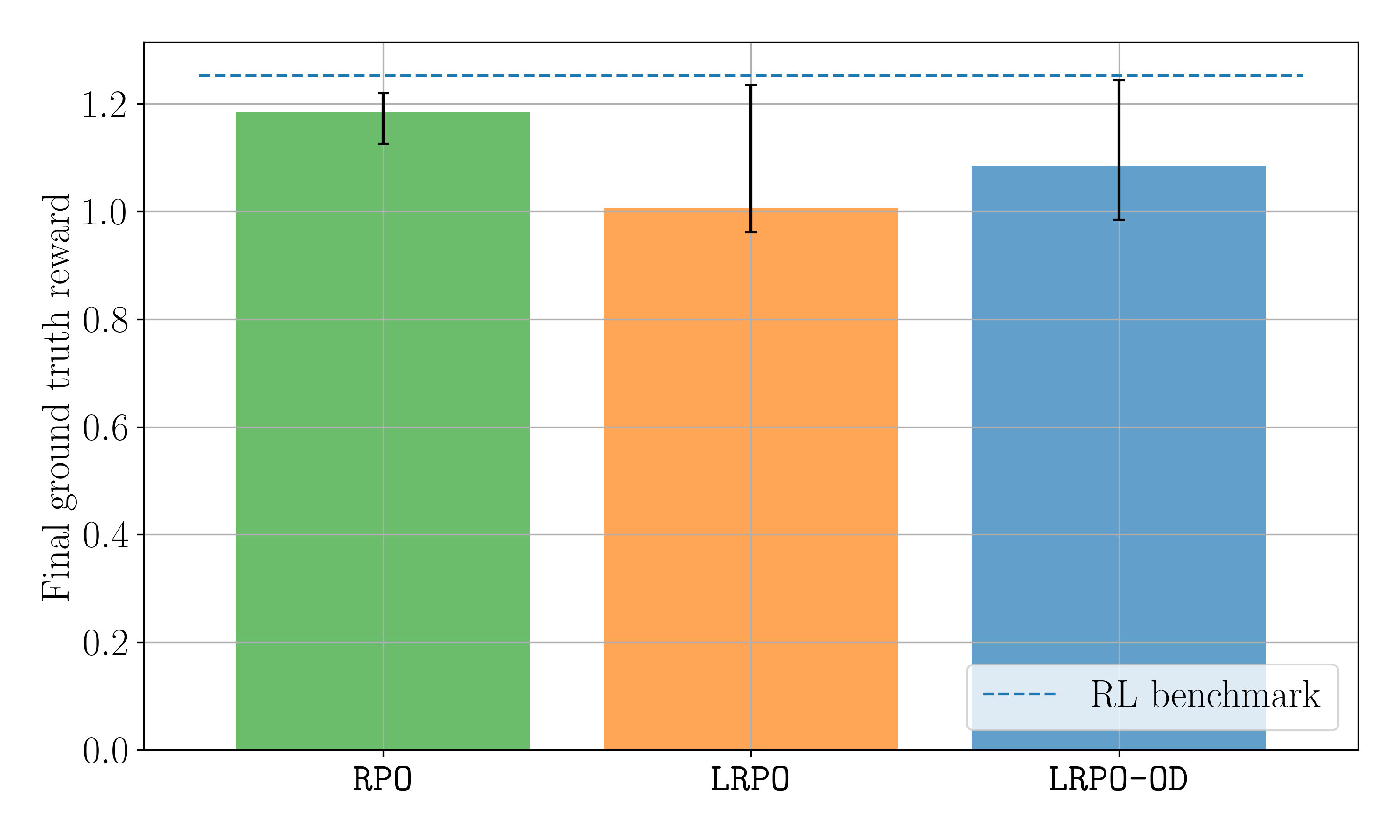}
    \label{fig:ground_truth_reward_pure_exploration}
    
    \vspace{-0.5cm}
    \hspace{0.5cm}(a)
\end{subfigure}
\hfill
\begin{subfigure}[t]{0.48\textwidth}
    \includegraphics[width=\textwidth]{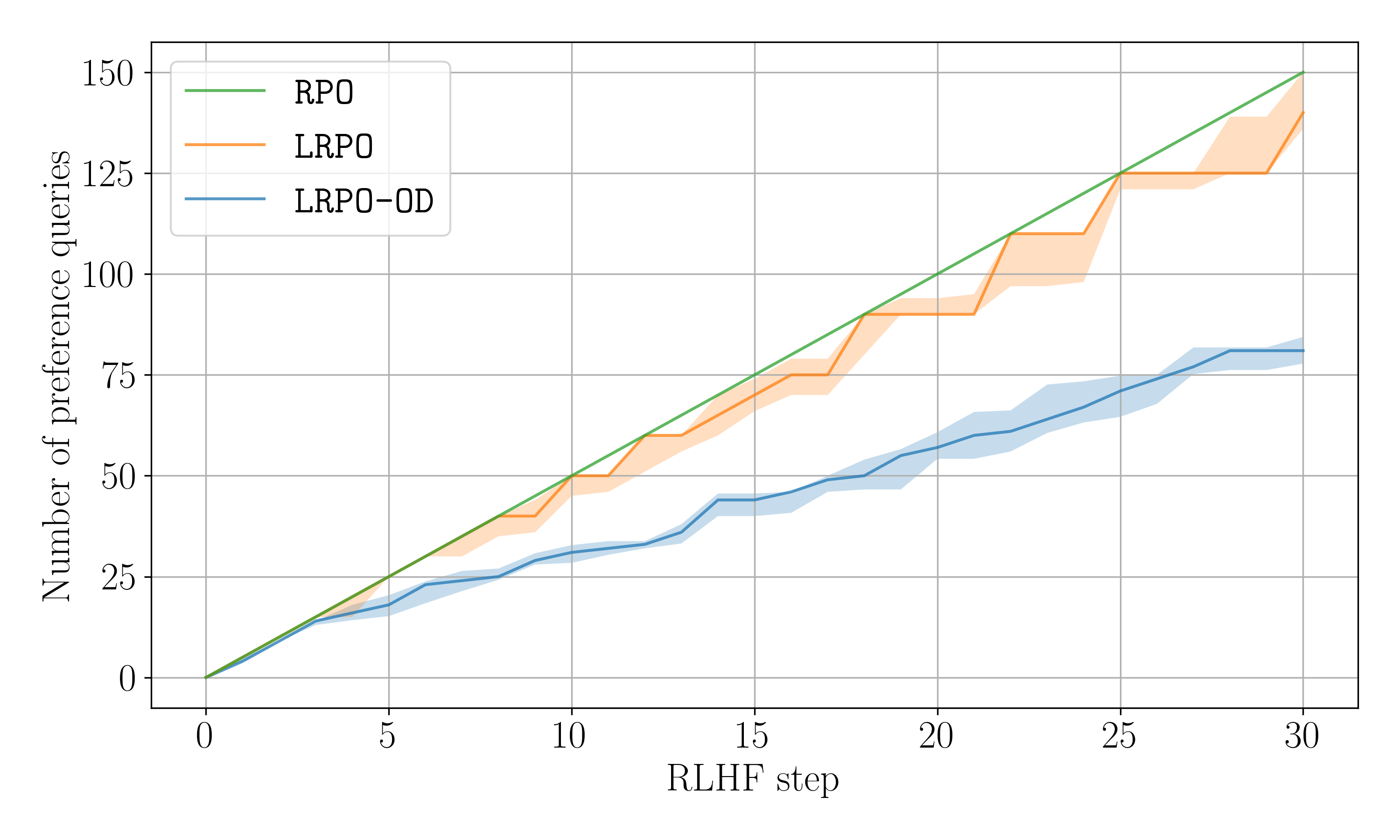}
    \label{fig:num_queries_pure_exploration}
    
    \vspace{-0.5cm}
    \hspace{3.3cm}(b)
\end{subfigure}
\caption{
Comparison of RLHF algorithms in terms of (a) the last iterate ground truth reward $V_{\theta^*}^{\pihat}$ (estimated from samples) and (b) number of preference queries performed. In particular, we compare \texttt{RPO-Explore} (green, Algorithm~\ref{alg:regret}) with its lazy versions \texttt{LRPO-Explore} and \texttt{LRPO-OD-Explore} (orange \& blue). Here, \texttt{LRPO-Explore} and \texttt{LRPO-OD-Explore} refer to Algorithm~\ref{alg:lazy_thompson_sampling} without and with optimal design subroutine. The error bars indicate the 0.2 and 0.8 quantiles, across 10 independent runs. The dashed blue line indicates the mean reward achieved by PPO with the ground truth parameter $\theta^*$. }
\label{fig:experiments_pure_exploration}
\end{figure}

\begin{figure}[h]
\centering
    \includegraphics[width=0.48\textwidth]{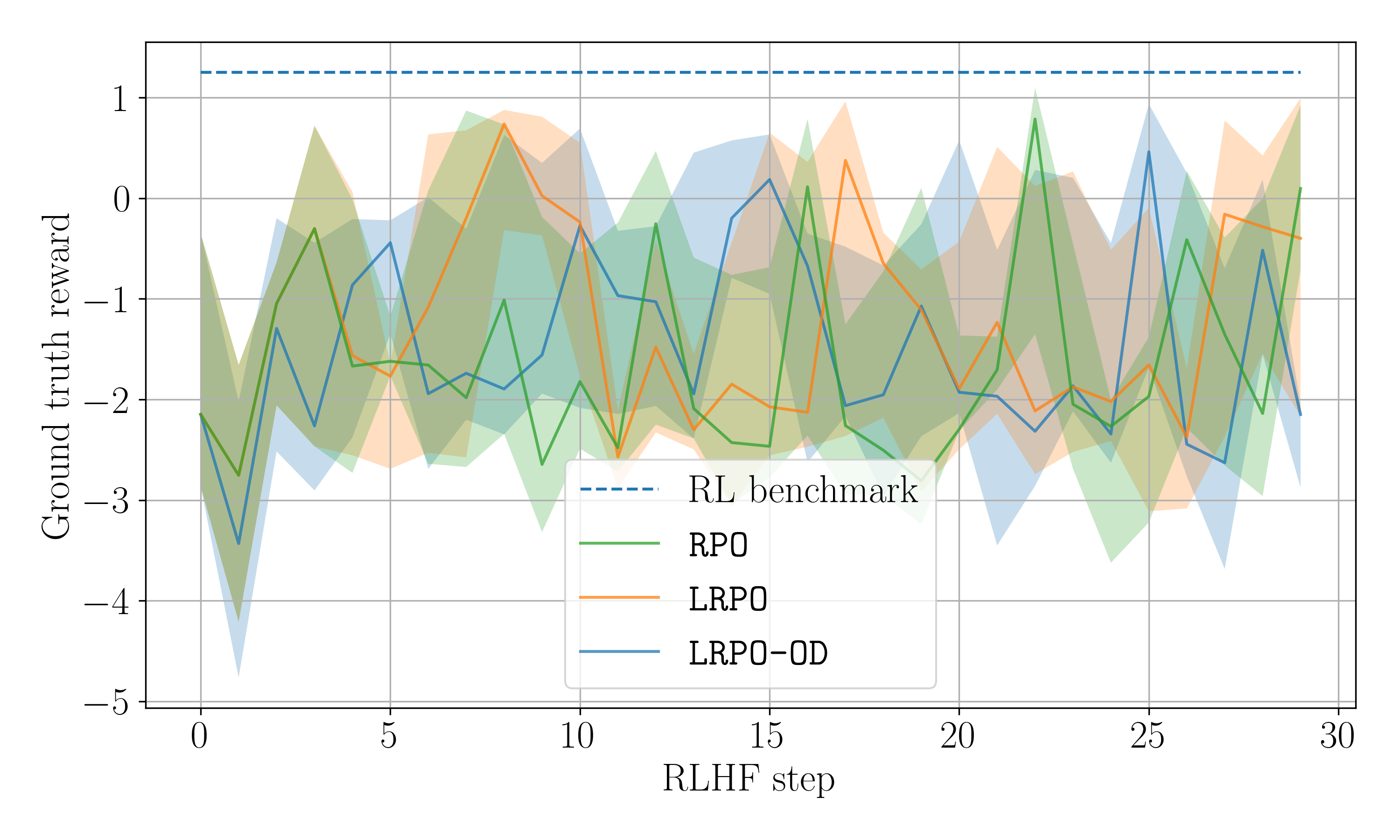}
\caption{Comparing the ground truth rewards $V_{\theta^*}^{\pi_t}$ (estimated from samples) of RLHF algorithms for reward-free exploration during training, using the same color codes as in Figure~\ref{fig:experiments_pure_exploration}.}
\label{fig:experiments_pure_exploration2}
\end{figure}

\newpage
\section*{NeurIPS Paper Checklist}

\begin{enumerate}

\item {\bf Claims}
    \item[] Question: Do the main claims made in the abstract and introduction accurately reflect the paper's contributions and scope?
    \item[] Answer: \answerYes{} 
    \item[] Justification: We provide our theoretical claims in Sections \ref{sec:randomized_preference_optimization} and \ref{sec:practical_algorithm}, we also provide our empirical results in Section \ref{sec:experiments}.
    \item[] Guidelines:
    \begin{itemize}
        \item The answer NA means that the abstract and introduction do not include the claims made in the paper.
        \item The abstract and/or introduction should clearly state the claims made, including the contributions made in the paper and important assumptions and limitations. A No or NA answer to this question will not be perceived well by the reviewers. 
        \item The claims made should match theoretical and experimental results, and reflect how much the results can be expected to generalize to other settings. 
        \item It is fine to include aspirational goals as motivation as long as it is clear that these goals are not attained by the paper. 
    \end{itemize}

\item {\bf Limitations}
    \item[] Question: Does the paper discuss the limitations of the work performed by the authors?
    \item[] Answer: \answerYes{}{} 
    \item[] Justification: We address the limitations of our work in Section \ref{sec:conclusion}.
    \item[] Guidelines:
    \begin{itemize}
        \item The answer NA means that the paper has no limitation while the answer No means that the paper has limitations, but those are not discussed in the paper. 
        \item The authors are encouraged to create a separate "Limitations" section in their paper.
        \item The paper should point out any strong assumptions and how robust the results are to violations of these assumptions (e.g., independence assumptions, noiseless settings, model well-specification, asymptotic approximations only holding locally). The authors should reflect on how these assumptions might be violated in practice and what the implications would be.
        \item The authors should reflect on the scope of the claims made, e.g., if the approach was only tested on a few datasets or with a few runs. In general, empirical results often depend on implicit assumptions, which should be articulated.
        \item The authors should reflect on the factors that influence the performance of the approach. For example, a facial recognition algorithm may perform poorly when image resolution is low or images are taken in low lighting. Or a speech-to-text system might not be used reliably to provide closed captions for online lectures because it fails to handle technical jargon.
        \item The authors should discuss the computational efficiency of the proposed algorithms and how they scale with dataset size.
        \item If applicable, the authors should discuss possible limitations of their approach to address problems of privacy and fairness.
        \item While the authors might fear that complete honesty about limitations might be used by reviewers as grounds for rejection, a worse outcome might be that reviewers discover limitations that aren't acknowledged in the paper. The authors should use their best judgment and recognize that individual actions in favor of transparency play an important role in developing norms that preserve the integrity of the community. Reviewers will be specifically instructed to not penalize honesty concerning limitations.
    \end{itemize}

\item {\bf Theory assumptions and proofs}
    \item[] Question: For each theoretical result, does the paper provide the full set of assumptions and a complete (and correct) proof?
    \item[] Answer: \answerYes{} 
    \item[] Justification: We provide our assumptions in Section \ref{sec:preliminaries} and our proofs in the appendix.
    \item[] Guidelines:
    \begin{itemize}
        \item The answer NA means that the paper does not include theoretical results. 
        \item All the theorems, formulas, and proofs in the paper should be numbered and cross-referenced.
        \item All assumptions should be clearly stated or referenced in the statement of any theorems.
        \item The proofs can either appear in the main paper or the supplemental material, but if they appear in the supplemental material, the authors are encouraged to provide a short proof sketch to provide intuition. 
        \item Inversely, any informal proof provided in the core of the paper should be complemented by formal proofs provided in appendix or supplemental material.
        \item Theorems and Lemmas that the proof relies upon should be properly referenced. 
    \end{itemize}

    \item {\bf Experimental result reproducibility}
    \item[] Question: Does the paper fully disclose all the information needed to reproduce the main experimental results of the paper to the extent that it affects the main claims and/or conclusions of the paper (regardless of whether the code and data are provided or not)?
    \item[] Answer: \answerYes{} 
    \item[] Justification: Our experimental environment will be publicly available, and all experimental details can be found in Appendix~\ref{app:sec:experiments}.
    \item[] Guidelines:
    \begin{itemize}
        \item The answer NA means that the paper does not include experiments.
        \item If the paper includes experiments, a No answer to this question will not be perceived well by the reviewers: Making the paper reproducible is important, regardless of whether the code and data are provided or not.
        \item If the contribution is a dataset and/or model, the authors should describe the steps taken to make their results reproducible or verifiable. 
        \item Depending on the contribution, reproducibility can be accomplished in various ways. For example, if the contribution is a novel architecture, describing the architecture fully might suffice, or if the contribution is a specific model and empirical evaluation, it may be necessary to either make it possible for others to replicate the model with the same dataset, or provide access to the model. In general. releasing code and data is often one good way to accomplish this, but reproducibility can also be provided via detailed instructions for how to replicate the results, access to a hosted model (e.g., in the case of a large language model), releasing of a model checkpoint, or other means that are appropriate to the research performed.
        \item While NeurIPS does not require releasing code, the conference does require all submissions to provide some reasonable avenue for reproducibility, which may depend on the nature of the contribution. For example
        \begin{enumerate}
            \item If the contribution is primarily a new algorithm, the paper should make it clear how to reproduce that algorithm.
            \item If the contribution is primarily a new model architecture, the paper should describe the architecture clearly and fully.
            \item If the contribution is a new model (e.g., a large language model), then there should either be a way to access this model for reproducing the results or a way to reproduce the model (e.g., with an open-source dataset or instructions for how to construct the dataset).
            \item We recognize that reproducibility may be tricky in some cases, in which case authors are welcome to describe the particular way they provide for reproducibility. In the case of closed-source models, it may be that access to the model is limited in some way (e.g., to registered users), but it should be possible for other researchers to have some path to reproducing or verifying the results.
        \end{enumerate}
    \end{itemize}

\item {\bf Open access to data and code}
    \item[] Question: Does the paper provide open access to the data and code, with sufficient instructions to faithfully reproduce the main experimental results, as described in supplemental material?
    \item[] Answer: \answerYes{} 
    \item[] Justification: We provide the code used for our experiment with the supplementary files.
    \item[] Guidelines:
    \begin{itemize}
        \item The answer NA means that paper does not include experiments requiring code.
        \item Please see the NeurIPS code and data submission guidelines (\url{https://nips.cc/public/guides/CodeSubmissionPolicy}) for more details.
        \item While we encourage the release of code and data, we understand that this might not be possible, so “No” is an acceptable answer. Papers cannot be rejected simply for not including code, unless this is central to the contribution (e.g., for a new open-source benchmark).
        \item The instructions should contain the exact command and environment needed to run to reproduce the results. See the NeurIPS code and data submission guidelines (\url{https://nips.cc/public/guides/CodeSubmissionPolicy}) for more details.
        \item The authors should provide instructions on data access and preparation, including how to access the raw data, preprocessed data, intermediate data, and generated data, etc.
        \item The authors should provide scripts to reproduce all experimental results for the new proposed method and baselines. If only a subset of experiments are reproducible, they should state which ones are omitted from the script and why.
        \item At submission time, to preserve anonymity, the authors should release anonymized versions (if applicable).
        \item Providing as much information as possible in supplemental material (appended to the paper) is recommended, but including URLs to data and code is permitted.
    \end{itemize}

\item {\bf Experimental setting/details}
    \item[] Question: Does the paper specify all the training and test details (e.g., data splits, hyperparameters, how they were chosen, type of optimizer, etc.) necessary to understand the results?
    \item[] Answer: \answerYes{} 
    \item[] Justification: The details are provided in Appendix~\ref{app:sec:experiments}.
    \item[] Guidelines:
    \begin{itemize}
        \item The answer NA means that the paper does not include experiments.
        \item The experimental setting should be presented in the core of the paper to a level of detail that is necessary to appreciate the results and make sense of them.
        \item The full details can be provided either with the code, in appendix, or as supplemental material.
    \end{itemize}

\item {\bf Experiment statistical significance}
    \item[] Question: Does the paper report error bars suitably and correctly defined or other appropriate information about the statistical significance of the experiments?
    \item[] Answer: \answerYes{} 
    \item[] Justification: Our experiment section provides shaded areas for the standard deviation of the performance curves over 10 independent training runs. Additional details can be found therein.
    \item[] Guidelines:
    \begin{itemize}
        \item The answer NA means that the paper does not include experiments.
        \item The authors should answer "Yes" if the results are accompanied by error bars, confidence intervals, or statistical significance tests, at least for the experiments that support the main claims of the paper.
        \item The factors of variability that the error bars are capturing should be clearly stated (for example, train/test split, initialization, random drawing of some parameter, or overall run with given experimental conditions).
        \item The method for calculating the error bars should be explained (closed form formula, call to a library function, bootstrap, etc.)
        \item The assumptions made should be given (e.g., Normally distributed errors).
        \item It should be clear whether the error bar is the standard deviation or the standard error of the mean.
        \item It is OK to report 1-sigma error bars, but one should state it. The authors should preferably report a 2-sigma error bar than state that they have a 96\% CI, if the hypothesis of Normality of errors is not verified.
        \item For asymmetric distributions, the authors should be careful not to show in tables or figures symmetric error bars that would yield results that are out of range (e.g. negative error rates).
        \item If error bars are reported in tables or plots, The authors should explain in the text how they were calculated and reference the corresponding figures or tables in the text.
    \end{itemize}

\item {\bf Experiments compute resources}
    \item[] Question: For each experiment, does the paper provide sufficient information on the computer resources (type of compute workers, memory, time of execution) needed to reproduce the experiments?
    \item[] Answer: \answerYes{} 
    \item[] Justification: See appendix~\ref{app:sec:experiments}.
    \item[] Guidelines:
    \begin{itemize}
        \item The answer NA means that the paper does not include experiments.
        \item The paper should indicate the type of compute workers CPU or GPU, internal cluster, or cloud provider, including relevant memory and storage.
        \item The paper should provide the amount of compute required for each of the individual experimental runs as well as estimate the total compute. 
        \item The paper should disclose whether the full research project required more compute than the experiments reported in the paper (e.g., preliminary or failed experiments that didn't make it into the paper). 
    \end{itemize}
    
\item {\bf Code of ethics}
    \item[] Question: Does the research conducted in the paper conform, in every respect, with the NeurIPS Code of Ethics \url{https://neurips.cc/public/EthicsGuidelines}?
    \item[] Answer: \answerYes{} 
    \item[] Justification: We don't see any conflict with the NeurIPS Code of Ethics.
    \item[] Guidelines:
    \begin{itemize}
        \item The answer NA means that the authors have not reviewed the NeurIPS Code of Ethics.
        \item If the authors answer No, they should explain the special circumstances that require a deviation from the Code of Ethics.
        \item The authors should make sure to preserve anonymity (e.g., if there is a special consideration due to laws or regulations in their jurisdiction).
    \end{itemize}

\item {\bf Broader impacts}
    \item[] Question: Does the paper discuss both potential positive societal impacts and negative societal impacts of the work performed?
    \item[] Answer: \answerNA{} 
    \item[] Justification: We do not foresee any negative societal impact for our algorithms.
    \item[] Guidelines:
    \begin{itemize}
        \item The answer NA means that there is no societal impact of the work performed.
        \item If the authors answer NA or No, they should explain why their work has no societal impact or why the paper does not address societal impact.
        \item Examples of negative societal impacts include potential malicious or unintended uses (e.g., disinformation, generating fake profiles, surveillance), fairness considerations (e.g., deployment of technologies that could make decisions that unfairly impact specific groups), privacy considerations, and security considerations.
        \item The conference expects that many papers will be foundational research and not tied to particular applications, let alone deployments. However, if there is a direct path to any negative applications, the authors should point it out. For example, it is legitimate to point out that an improvement in the quality of generative models could be used to generate deepfakes for disinformation. On the other hand, it is not needed to point out that a generic algorithm for optimizing neural networks could enable people to train models that generate Deepfakes faster.
        \item The authors should consider possible harms that could arise when the technology is being used as intended and functioning correctly, harms that could arise when the technology is being used as intended but gives incorrect results, and harms following from (intentional or unintentional) misuse of the technology.
        \item If there are negative societal impacts, the authors could also discuss possible mitigation strategies (e.g., gated release of models, providing defenses in addition to attacks, mechanisms for monitoring misuse, mechanisms to monitor how a system learns from feedback over time, improving the efficiency and accessibility of ML).
    \end{itemize}
    
\item {\bf Safeguards}
    \item[] Question: Does the paper describe safeguards that have been put in place for responsible release of data or models that have a high risk for misuse (e.g., pretrained language models, image generators, or scraped datasets)?
    \item[] Answer: \answerNA{} 
    \item[] Justification: We do not release models or datasets in this paper.
    \item[] Guidelines:
    \begin{itemize}
        \item The answer NA means that the paper poses no such risks.
        \item Released models that have a high risk for misuse or dual-use should be released with necessary safeguards to allow for controlled use of the model, for example by requiring that users adhere to usage guidelines or restrictions to access the model or implementing safety filters. 
        \item Datasets that have been scraped from the Internet could pose safety risks. The authors should describe how they avoided releasing unsafe images.
        \item We recognize that providing effective safeguards is challenging, and many papers do not require this, but we encourage authors to take this into account and make a best faith effort.
    \end{itemize}

\item {\bf Licenses for existing assets}
    \item[] Question: Are the creators or original owners of assets (e.g., code, data, models), used in the paper, properly credited and are the license and terms of use explicitly mentioned and properly respected?
    \item[] Answer: \answerYes{} 
    \item[] Justification: We are citing and crediting the Isaac Lab creators for their RL environment.
    \item[] Guidelines:
    \begin{itemize}
        \item The answer NA means that the paper does not use existing assets.
        \item The authors should cite the original paper that produced the code package or dataset.
        \item The authors should state which version of the asset is used and, if possible, include a URL.
        \item The name of the license (e.g., CC-BY 4.0) should be included for each asset.
        \item For scraped data from a particular source (e.g., website), the copyright and terms of service of that source should be provided.
        \item If assets are released, the license, copyright information, and terms of use in the package should be provided. For popular datasets, \url{paperswithcode.com/datasets} has curated licenses for some datasets. Their licensing guide can help determine the license of a dataset.
        \item For existing datasets that are re-packaged, both the original license and the license of the derived asset (if it has changed) should be provided.
        \item If this information is not available online, the authors are encouraged to reach out to the asset's creators.
    \end{itemize}

\item {\bf New assets}
    \item[] Question: Are new assets introduced in the paper well documented and is the documentation provided alongside the assets?
    \item[] Answer: \answerYes{}  
    \item[] Justification: Instructions for reproducing the experiments are provided in the readme.
    \item[] Guidelines:
    \begin{itemize}
        \item The answer NA means that the paper does not release new assets.
        \item Researchers should communicate the details of the dataset/code/model as part of their submissions via structured templates. This includes details about training, license, limitations, etc. 
        \item The paper should discuss whether and how consent was obtained from people whose asset is used.
        \item At submission time, remember to anonymize your assets (if applicable). You can either create an anonymized URL or include an anonymized zip file.
    \end{itemize}

\item {\bf Crowdsourcing and research with human subjects}
    \item[] Question: For crowdsourcing experiments and research with human subjects, does the paper include the full text of instructions given to participants and screenshots, if applicable, as well as details about compensation (if any)? 
    \item[] Answer: \answerNA{} 
    \item[] Justification: Our research is not involving any experiments with humans.
    \item[] Guidelines:
    \begin{itemize}
        \item The answer NA means that the paper does not involve crowdsourcing nor research with human subjects.
        \item Including this information in the supplemental material is fine, but if the main contribution of the paper involves human subjects, then as much detail as possible should be included in the main paper. 
        \item According to the NeurIPS Code of Ethics, workers involved in data collection, curation, or other labor should be paid at least the minimum wage in the country of the data collector. 
    \end{itemize}

\item {\bf Institutional review board (IRB) approvals or equivalent for research with human subjects}
    \item[] Question: Does the paper describe potential risks incurred by study participants, whether such risks were disclosed to the subjects, and whether Institutional Review Board (IRB) approvals (or an equivalent approval/review based on the requirements of your country or institution) were obtained?
    \item[] Answer: \answerNA{} 
    \item[] Justification: Our research is not involving any experiments with humans.
    \item[] Guidelines:
    \begin{itemize}
        \item The answer NA means that the paper does not involve crowdsourcing nor research with human subjects.
        \item Depending on the country in which research is conducted, IRB approval (or equivalent) may be required for any human subjects research. If you obtained IRB approval, you should clearly state this in the paper. 
        \item We recognize that the procedures for this may vary significantly between institutions and locations, and we expect authors to adhere to the NeurIPS Code of Ethics and the guidelines for their institution. 
        \item For initial submissions, do not include any information that would break anonymity (if applicable), such as the institution conducting the review.
    \end{itemize}

\item {\bf Declaration of LLM usage}
    \item[] Question: Does the paper describe the usage of LLMs if it is an important, original, or non-standard component of the core methods in this research? Note that if the LLM is used only for writing, editing, or formatting purposes and does not impact the core methodology, scientific rigorousness, or originality of the research, declaration is not required.
    \item[] Answer: \answerNo{} 
    \item[] Justification: We only use LLMs to improve the writing quality.
    \item[] Guidelines:
    \begin{itemize}
        \item The answer NA means that the core method development in this research does not involve LLMs as any important, original, or non-standard components.
        \item Please refer to our LLM policy (\url{https://neurips.cc/Conferences/2025/LLM}) for what should or should not be described.
    \end{itemize}

\end{enumerate}

\end{document}